\documentclass[sigconf]{acmart}

\usepackage{algorithm}
\usepackage{algorithmic}

\usepackage{threeparttable}
\usepackage{multirow}
\usepackage{booktabs}
\usepackage{adjustbox}
\usepackage{wasysym} 
\usepackage{subcaption}
\usepackage{soul}
\usepackage{amsmath}
\usepackage{colortbl}
\usepackage{xcolor}
\usepackage{enumitem}
\newcommand{\cmark}{\CIRCLE}
\newcommand{\xmark}{\Circle}

\newtheorem{prop}{Proposition}
\newtheorem{lemma}{Lemma}
\AtBeginDocument{%
  }

\copyrightyear{2026}
\acmYear{2026}
\setcopyright{cc}
\setcctype{by}
\acmConference[KDD 2026] {Proceedings of the 32nd ACM SIGKDD Conference on Knowledge Discovery and Data Mining V.1}{August 9--13, 2026}{Jeju Island, Republic of Korea.}
\acmBooktitle{Proceedings of the 32nd ACM SIGKDD Conference on Knowledge Discovery and Data Mining V.1 (KDD 2026), August 9--13, 2026, Jeju Island, Republic of Korea}
\acmISBN{979-8-4007-2258-5/2026/08}
\acmDOI{10.1145/3770854.3780322}

\acmSubmissionID{1138}

\begin{document}

\title{The Eminence in Shadow: Exploiting Feature Boundary Ambiguity for Robust Backdoor Attacks}

\author{Zhou Feng}
\orcid{0009-0006-5301-7019}
\affiliation{%
  \institution{Zhejiang University}
  \department{College of Computer Science and Technology}
  \city{Hangzhou}
  \country{China}
}
\email{zhou.feng@zju.edu.cn}

\author{Jiahao Chen}
\orcid{0000-0002-5894-662X}
\affiliation{%
  \institution{Zhejiang University}
  \department{College of Computer Science and Technology}
  \city{Hangzhou}
  \country{China}
}
\email{xaddwell@zju.edu.cn}

\author{Chunyi Zhou}
\orcid{0000-0003-0081-0946}
\affiliation{%
  \institution{Zhejiang University}
  \department{College of Computer Science and Technology}
  \city{Hangzhou}
  \country{China}
}
\email{zhouchunyi@zju.edu.cn}

\author{Yuwen Pu}
\orcid{0000-0003-2311-4943}
\affiliation{%
  \institution{Chongqing University}
  \department{School of Big Data \& Software Engineering}
  \city{Chongqing}
  \country{China}
}
\email{yw.pu@cqu.edu.cn}

\author{Tianyu Du}
\orcid{0000-0003-0896-0690}
\affiliation{%
  \institution{Zhejiang University}
  \department{School of Software Technology}
  \city{Hangzhou}
  \country{China}
}
\email{zjradty@zju.edu.cn}

\author{Jinbao Li}
\orcid{0000-0002-2432-8807}
\authornotemark[1]
\affiliation{%
  \institution{Qilu University of Technology (Shandong Academy of Science)}
  \city{Jinan}
  \country{China}
}
\email{lijb@qlu.edu.cn}

\author{Jianhai Chen}
\orcid{0000-0003-3524-3443}
\affiliation{%
  \institution{Zhejiang University}
  \department{College of Computer Science and Technology}
  \city{Hangzhou}
  \country{China}
}
\email{chenjh919@zju.edu.cn}

\author{Shouling Ji}
\orcid{0000-0003-4268-372X}
\authornote{Corresponding authors.}
\affiliation{%
  \institution{Zhejiang University}
  \department{College of Computer Science and Technology}
  \city{Hangzhou}
  \country{China}
}
\email{sji@zju.edu.cn}

\renewcommand{\shortauthors}{Zhou Feng et al.}
\acmArticleType{Review}
\acmCodeLink{https://github.com/borisveytsman/acmart}
\acmDataLink{htps://zenodo.org/link}
\acmContributions{BT and GKMT designed the study; LT, VB, and AP
  conducted the experiments, BR, HC, CP and JS analyzed the results,
  JPK developed analytical predictions, all authors participated in
  writing the manuscript.}



\begin{abstract}
Deep neural networks (DNNs) underpin critical applications yet remain vulnerable to backdoor attacks, typically reliant on heuristic brute-force methods. Despite significant empirical advancements in backdoor research, the lack of rigorous theoretical analysis limits understanding of underlying mechanisms, constraining attack predictability and adaptability. 
Therefore, we provide a theoretical analysis targeting backdoor attacks, focusing on how sparse decision boundaries enable disproportionate model manipulation.
Based on this finding, we derive a closed-form ``\textit{\textbf{ambiguous boundary region}}'' wherein negligible relabeled samples induce substantial misclassification. Influence function analysis further quantifies significant parameter shifts caused by these margin samples, with minimal impact on clean accuracy, formally grounding why such low poison rates suffice for efficacious attacks.
Leveraging these insights, we propose \textit{Eminence}, an explainable and robust black-box backdoor framework with provable theoretical guarantees and inherent stealth properties. \textit{Eminence} optimizes a universal, visually subtle trigger that strategically exploits vulnerable decision boundaries and effectively achieves robust misclassification with \textit{\textbf{exceptionally low poison rates}} ($\leq 0.01\%$, compared to SOTA methods typically requiring $\geq 1 \%$). Comprehensive experiments validate our theoretical discussions and demonstrate the effectiveness of \textit{Eminence}, confirming an exponential relationship between margin poisoning and adversarial boundary manipulation. 
\textit{Eminence} maintains $\geq 90\%$ attack success rate, exhibits negligible clean-accuracy loss, and demonstrates high transferability across diverse models, datasets and scenarios.
Our code is available at \url{https://github.com/NESA-Lab/Eminence}
\end{abstract}

\maketitle

\section{Introduction}

Deep neural networks (DNNs) have become the de facto foundation of modern vision and biometric services, ranging from autonomous driving~\cite{Yurtsever2020Survey, Ravindran2021MODreview} to face recognition access control~\cite{Nagrath2021SSDMNV2, Balaban2015DLFace}, yet they remain alarmingly vulnerable to backdoor attacks~\cite{gu2017badnets, Li2024BackdoorSurvey}. In such attacks an adversary injects a small fraction of poisoned samples into the training pipeline. Once the model is deployed, any input stamped with a secret trigger is forcibly misclassified while ordinary inputs are handled correctly. The consequence could be dramatic: an attacker can unlock a face recognition door with a printed sticker or misroute autonomous drones with an imperceptible pattern, all without arousing suspicion~\cite{chen2024rethinking, Huang2025IdentityLock}.


Although the field has seen rapid empirical advances~\cite{zeng2023narcissus, qi2023revisiting}, most backdoor attacks remain \textbf{engineering-driven} rather than \textbf{theory-driven}.  Current triggers are largely handcrafted or tuned by exhaustive search, and must be mixed into the training set at non-trivial rates (typically $\geq 1 \%$ rate) to attain reliable success \cite{nguyen2021wanet, liu2020reflection}. Due to the absence of a principled theoretical framework explaining how poisoned samples reshape decision boundaries, researchers continue to treat the boundary as a black box: iteratively testing triggers until achieving successful adversarial manipulation, at which point empirical results are reported. This brute-force routine carries three practical limitations:

\begin{table*}[t]
\centering
\caption{
Summary of backdoor attacks evaluated in this paper. 
\textit{Detection Robust} and \textit{Mitigation Robust} indicate the resistance to input-based detection and model-based mitigation, respectively. 
\textit{Poison Rate} denotes the minimum poisoning rate required to achieve an attack success rate exceeding 90\%. \xmark~The item is not supported by the attack; \cmark~The item is supported by the attack.
}
\label{tab:attack_updated}
\renewcommand{\arraystretch}{1.2}
\begin{adjustbox}{width=0.80\textwidth}
\begin{tabular}{lccccccc}
\toprule
\multirow{2}{*}{\textbf{Attack}} & 
\multicolumn{3}{c}{\textbf{Attack Properties}} & 
\multicolumn{2}{c}{\textbf{Robustness}} & 
\multicolumn{2}{c}{\textbf{Stealthiness}} \\
\cmidrule(lr){2-4} \cmidrule(lr){5-6} \cmidrule(lr){7-8}
& Clean-Label & Model-Agnostic & Data-Free & Detection Robust & Mitigation Robust & Invisible Trigger & Poison Rate \\
\midrule
BadNets~\cite{gu2017badnets}             & \xmark & \cmark & \cmark & \xmark & \xmark & \xmark & 7\% \\
Blended~\cite{chen2017targeted}             & \xmark & \cmark & \cmark & \xmark & \xmark & \cmark & 7\% \\
Refool~\cite{liu2020reflection}              & \cmark & \cmark & \cmark & \xmark & \xmark & \xmark & 0.57\% \\
LabelConsistent~\cite{turner2019label}     & \cmark & \cmark & \xmark & \cmark & \xmark & \cmark & 0.40\% \\
TUAP~\cite{zhao2020clean}                & \cmark & \xmark & \xmark & \cmark & \xmark & \cmark & 0.30\% \\
PhysicalBA~\cite{li2021physical}          & \xmark & \cmark & \cmark & \xmark & \xmark & \xmark & 0.50\% \\
WaNet~\cite{nguyen2021wanet}               & \cmark & \xmark & \cmark & \cmark & \cmark & \cmark & 10\% \\
AdaptivePatch~\cite{qi2023revisiting} & \xmark & \cmark & \cmark & \xmark & \xmark & \xmark & 0.30\% \\
Narcissus~\cite{zeng2023narcissus}          & \cmark & \xmark & \cmark & \cmark & \cmark & \cmark & 0.05\% \\
\rowcolor[HTML]{FFDADA} 
\textbf{\textit{Eminence} (Ours)}       & \cmark & \cmark & \cmark & \cmark & \cmark & \cmark & $\leq$ 0.01\% \\
\bottomrule
\end{tabular}
\end{adjustbox}

\end{table*}

\begin{itemize}
    \item \textbf{Shallow insight into model fragility.} Existing methods treat model boundaries as black boxes, blindly optimizing triggers via brute-force search and neglecting the geometric structure of the boundary regions, which limits their ability to craft efficient and explainable attacks.
    \item \textbf{Inefficient poison budgets.} Without an analytical characterization of influential regions, current attacks indiscriminately inject large numbers of poisoned samples into densely populated feature regions, which inflates the poisoning cost, increases detectability, and unnecessarily compromises clean accuracy.
    \item  \textbf{Poor cross-model transferability.} Existing strategies are highly tailored to specific victim models, neglecting theory-based properties of model boundaries that generalize across architectures. Consequently, triggers tuned through brute-force optimization seldom remain effective when deployed against unseen victim architectures, forcing adversaries into costly retraining cycles.
\end{itemize}


To this end, this work is dedicated to investigating backdoors through the lens of decision boundary theory.
By tracing thousands of training trajectories across samples, we find that the sparsely populated margin separating class clusters is far more malleable than the dense interiors: 
\textbf{A few re-labelled margin points can shift the decision boundary across classes with minimal impact on the overall model.}
We capture this phenomenon with three complementary analyses, also illustrated in Figure~\ref{fig:insights}:

\begin{figure}[]
  \centering
  \includegraphics[width=0.47\textwidth]{figure/insights.pdf}
  \caption{
  Insights underpinning our attack.
  (a) Ambiguous Margin: a thin, low-density region where minimal relabelling can shift the boundary. 
  (b) Gradient Amplification: relabelled margin samples induce large parameter updates.
  (c) Boundary Absorption: a few poisons suffice to absorb the boundary and achieve high attack success.
}
  \label{fig:insights}
  \vspace{-1em}
\end{figure}

\begin{itemize}
    \item[$\blacksquare$] \textbf{Closed-form characterization of the ambiguous margin.} Using class prototypes extracted from a frozen network, we delineate a thin ``band'' that contains a negligible fraction of training features yet accounts for disproportionate misclassifications when labels are flipped in situ. This band proves architecture-agnostic and persists after fine-tuning, making it an attractive universal target.
    \item[$\blacksquare$] \textbf{Gradient-amplification study via influence functions.} We measure how a single poisoned, re-labelled margin sample tilts the loss landscape. The resulting parameter drift is an order of magnitude larger than that caused by an interior point of equal norm, confirming that the margin enjoys a natural leverage multiplier while its effect on clean accuracy remains first-order small.
    \item[$\blacksquare$] \textbf{Absorption bound confirmed in practice.} Controlled poisoning experiments show an almost log-linear relationship between the number of margin poisons and the probability that all margin points are reassigned to the attacker’s target class; five to ten poisons push this probability above $95 \%$, echoing the exponential trend predicted by analysis.
\end{itemize}

Building on these insights, we design an explainable, highly-stealth attack, turning theory into practice. Key contributions are:

\begin{enumerate} 
    \item We provide a formal analysis that links re-labelled margin samples to an exponential expansion of the target decision region while provably bounding clean-accuracy loss, where unifies dirty- and clean-label settings. 
    \item We present \textit{Eminence}, 
    an explainable boundary-seeking backdoor pipeline that learns an imperceptible trigger to collapse features into the multi-class margin with few re-labeling, and works out-of-the-box on white-, gray- and black-box victims.  
    \item We conduct extensive experiments on multiple benchmark datasets across six mainstream architectures. \textit{Eminence} attains $\ge 90\%$ attack success, incurs $<0.5\%$ clean-accuracy drop, and maintains high transferability to unseen models. Ablation studies corroborate the predicted leverage of margin poisons and the robustness of the learned trigger.
    \item To foster follow-up research, we release a clean, reproducible codebase \footnote{\url{https://github.com/NESA-Lab/Eminence}.} that includes all theoretical routines, trigger optimization scripts, and evaluation pipelines.

\end{enumerate}

\section{Preliminary \& Related Work}

\subsection{Backdoor Training}
\label{subsec:backdoor_training}


Backdoor training is a malicious paradigm in which an adversary implants hidden behaviors into a machine learning model by manipulating the training data. 
Generally,
let $\mathcal{D}_{\text{benign}} = \{(x_i, y_i)\}_{i=1}^{N}$ denote the clean training dataset used to train a model $f_\theta$ with parameters $\theta$. The attacker generates a poisoned dataset $\mathcal{D}_{\text{poison}} = \{(x_j', y_j')\}_{j=1}^{M}$, where each $x_j' = t(x_j;\varphi)$ contains an embedded trigger $t(\cdot;\varphi)$, and $y_j'$ is determined according to the attack type. The complete training set becomes $\mathcal{D}_{\text{train}} = \mathcal{D}_{\text{benign}} \cup \mathcal{D}_{\text{poison}}$, and the objective is to learn a model $f_{\theta'}$ such that $f_{\theta'}(x) \approx y$ for clean inputs, yet $f_{\theta'}(x') = y'$ for inputs with the trigger.

Depending on how the poisoned samples are labeled, backdoor attacks are typically classified into two categories:


\begin{itemize}
    \item \textit{Dirty-label Attack.} In this setting, poisoned samples are intentionally mislabeled to align with the attacker’s target label, i.e., $y' \ne y$. This approach enforces a strong association between the injected trigger and the target class, typically resulting in high attack success rates. 
    \item \textit{Clean-label Attack.} Here, the poisoned samples retain their original labels, i.e., $y' = y$, making the attack significantly more stealthy. To succeed, the trigger must induce feature collisions with the target class in a semantically meaningful way under normal labeling. Although this setting is more inconspicuous, it typically requires more sophisticated design to maintain high efficacy.
\end{itemize}

\subsection{Backdoor Attacks}

Backdoor attacks embed a hidden mapping in the model that activates on trigger presence. They are mainly categorized into \textbf{model-supply} and \textbf{poisoning-based} attacks.

\textbf{Model-supply attacks} assume full attacker control over training. The adversary embeds backdoors before releasing the model, often under the guise of open-source utilities. Early works manipulate weights~\cite{dumford2020backdooring, rakin2020tbt, garg2020adversarial}, while recent data-free methods~\cite{cao2024datafree, lv2023datafree} use surrogate data and neuron rewiring. Others exploit structure modifications~\cite{tang2020simple, qi2022practical, li2021deeppayload}, enabling deployment-stage backdoors. These attacks avoid training pipeline access but rely on control of model release and may leave detectable traces.

\textbf{Poisoning-based attacks} inject poisoned samples into the victim’s training set. BadNets~\cite{gu2017badnets} introduced this paradigm with simple triggers. Subsequent methods improved stealth via clean-label design~\cite{chen2017targeted, turner2019label, liu2020reflection, zhao2020clean, nguyen2021wanet, feng2025poison}, or by enhancing robustness in physical settings~\cite{li2021physical}. Advanced techniques suppress feature separability~\cite{qi2023revisiting}, or use only target-class and public data to achieve high ASR with minimal poison rate~\cite{zeng2023narcissus, saha2020hidden, souri2022sleeper}. These attacks are stealthy and flexible, but rely on data injection capability during training.

\subsection{Backdoor Defenses}
\label{subsec:backdoor_defense}
Defenses could be grouped into \textbf{input-based detection} and \textbf{model-based mitigation}.

\textbf{Input-based detection} identifies trigger-carrying inputs at inference. STRIP~\cite{gao2019strip, gao2022multidomain} and SCALE-UP~\cite{guo2023scaleup} detect perturbation sensitivity and prediction consistency, respectively. MSPC~\cite{pal2024backdoor} improves detection with mask-aware optimization. Beatrix~\cite{ma2023beatrix} leverages high-order Gram correlations for universal and sample-specific backdoor detection. IBD-PSC~\cite{hou2024ibdpsc} enhances generalization via confidence-based batch normalization probing, addressing limitations in~\cite{chou2018sentinet, liu2023detecting}.

\textbf{Model-based mitigation} focuses on repairing or immunizing the model. Neural Cleanse~\cite{wang2019neural} detects suspicious patterns via anomaly scoring. NAD~\cite{li2021neural} uses attention distillation to purge triggers. I-BAU~\cite{zeng2022adversarial} solves a minimax unlearning problem via implicit gradients. FT-SAM~\cite{zhu2023enhancing} perturbs sensitive neurons through sharpness-aware fine-tuning~\cite{foret2021sharpness}. ABL~\cite{li2021antibackdoor} isolates poisoned samples during training using gradient ascent. Recent proactive approaches~\cite{qi2023proactive, pu2024mellivora} show that backdoor resistance can be trained directly from poisoned data, without clean supervision.

\section{Threat Model}


We examine a realistic backdoor threat scenario where an adversary is capable of poisoning the training pipeline of a victim model, a setting commonly observed in large-scale data collection~\cite{Roh2021DataCollectionSurvey}, MLaaS supply chains~\cite{Ribeiro2015MLaaS}, or outsourced model training~\cite{Hong2022OpenSourceSampling}.
The adversary's objective is to induce targeted misclassification upon the presence of a trigger pattern \textbf{with a poisoning rate significantly lower than that of SOTA methods}, while ensuring that the victim model maintains high accuracy on clean inputs. A successful attack enables the adversary to stealthily exploit or misuse the post-deployment model.
As illustrated in Figure~\ref{fig:threat_model}, we explicitly specify the adversary's knowledge and capabilities, along with the possible countermeasures the victim may deploy.

\begin{figure}[]
  \centering
  \includegraphics[width=0.33\textwidth]{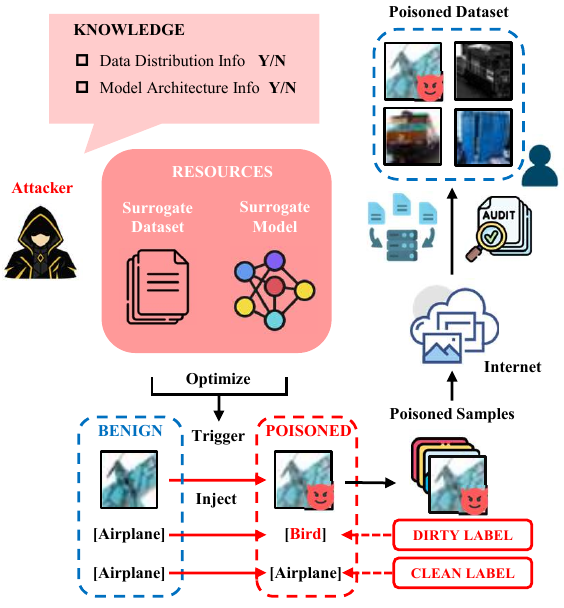}
  \caption{
    Threat Model. The adversary optimizes a trigger with resources, 
    injects clean- or dirty-label poisoned samples into the training pipeline, 
    and the victim unknowingly learns a backdoored model. 
    Prior knowledge of data or model (\textit{Y/N}) covers white-, gray-, and black-box scenarios.
    }
  \label{fig:threat_model}
  \vspace{-1em}
\end{figure}

\subsection{Victim Assumption}

The victim is typically a resource-rich entity, such as a cloud service provider, industrial laboratory, or public dataset curator. In practice, these victims often aggregate data from untrusted or semi-trusted sources (e.g., open data scraping~\cite{Khder2021WebScraping}, third-party uploads~\cite{Libert2018AuditPrivacyPolicies}, or crowdsourced annotation~\cite{Diaz2022CrowdWorksheets}). And during data ingestion and model training, the victim likely employs input-based defenses offered by the MLaaS system or self-deployed to identify and filter out suspicious or anomalous samples~\cite{Ribeiro2015MLaaS, Hanzlik2021MLCapsule}. And after training, the victim could also further scrutinize the resulting model by performing feature space analyses, such as clustering or outlier detection, to uncover potential malicious behaviors or backdoor triggers. In addition, the victim could utilize a combination of static and dynamic model analyses to further enhance the reliability of the deployed model. Such strategies may include monitoring and regularizing neuron activations, aligning model attention patterns with those of trusted references, and leveraging adversarial training or fine-tuning techniques designed to actively remove potential backdoor behaviors. By systematically applying these post-training analyses and adjustments, the victim aims to mitigate hidden threats and reinforce the model's robustness against backdoor attacks. These defensive capabilities collectively increase the difficulty for a successful backdoor attack and must be considered in the threat model.

\subsection{Adversary Assumption}
\label{subsec:adversary_assumption}
We consider an adversary capable of operating in white-box, gray-box, and black-box settings, reflecting the escalating knowledge levels achievable in real-world attacks. Our backdoor attack, \textit{Eminence} could achieve cross-scenario universality through our design.

\noindent\textbf{Adversary Knowledge.}
We categorize the adversary's knowledge into three distinct settings:

\begin{itemize}
    \item \textbf{White-box:} 
    The adversary has full access to the victim’s model architecture and data distribution, as in collaborative training settings~\cite{Predd2009DistributedLearning}. But All manipulations are restricted to data poisoning; no training procedure is altered.

    \item \textbf{Gray-box:} 
    The adversary knows the data domain and task, and can approximate the model family but lacks access to the precise architecture, which reflects transfer attacks or model outsourcing scenarios commonly encountered in MLaaS pipelines~\cite{Weiss2016TransferLearningSurvey}.

    \item \textbf{Black-box:} 
    The adversary lacks direct knowledge of the victim’s task, model, or data, and must rely solely on publicly available surrogates and synthetic data~\cite{Li2024BackdoorSurvey}. This data-free scenario is the most restrictive and generalizable case.
\end{itemize}

\noindent\textbf{Adversary Capability.}
Across all scenarios, the adversary is restricted to injecting poisoned samples into the training data collection process and does not control the training procedure or model selection. We assume that the adversary possesses a small local surrogate dataset $\mathcal{D}_{\text{atk}}$ and a pretrained surrogate model $f_{\text{atk}}$, which are used to optimize a trigger pattern $t$ in the feature space. The precise relationship between $(\mathcal{D}_{\text{atk}}, f_{\text{atk}})$ and the victim’s $(\mathcal{D}_{\text{benign}}, f)$ is characterized as follows:
\begin{itemize}
    \item \textbf{White-box:} $\mathcal{D}_{\text{atk}} \sim \mathcal{D}_{\text{benign}}$ \& $f_{\text{atk}} \equiv f$
    \item \textbf{Gray-box:} $\mathcal{D}_{\text{atk}} \sim \mathcal{D}_{\text{benign}}$ \&  $f_{\text{atk}} \not\equiv f$
    \item \textbf{Black-box:} $\mathcal{D}_{\text{atk}} \not\sim \mathcal{D}_{\text{benign}}$ \& $f_{\text{atk}} \not\equiv f$
\end{itemize}
where $\sim$ denotes ``drawn from the same distribution'' and $\equiv$ denotes ``identical architecture''. Importantly, the adversary’s local dataset is assumed much smaller than the victim’s, i.e., $|\mathcal{D}_{\text{atk}}| \ll |\mathcal{D}_{\text{benign}}|$.

Leveraging these limited resources, the adversary aims to optimize a universal trigger pattern $t(\cdot;\varphi)$ in the feature space, and constructs a poisoned dataset $\mathcal{D}_{\text{poison}} = \{(x_j', y_j')\}_{j=1}^{M}$, where $x_j' = t(x_j;\varphi)$. The poisoned samples are then injected into the overall training set $\mathcal{D}_{\text{train}} = \mathcal{D}_{\text{benign}} \cup \mathcal{D}_{\text{poison}}$ by means such as contributing to third-party or crowdsourced datasets, or by directly infiltrating the victim’s data collection pipeline.

\section{Method}

\subsection{Problem Definition}
\label{subsec:problem_def}

We aim to design a backdoor attack that is \textbf{both stealthy and effective} under practical constraints, with three desiderata:  
(i) \textbf{universality} across dirty‑ and clean‑label settings,  
(ii) \textbf{strong trigger transferability} in black‑box environments, and  
(iii) \textbf{low accuracy loss} on clean data.

\noindent\textbf{Ambiguous Boundary Region.}
Let 
$F_\theta:\mathbb{R}^{d_x}\!\to\!\mathbb{R}^{d_f}$
be the feature extractor, $ z = F_\theta(x)$ be the features of sample and  $g_\theta(z)=\arg\max_{c\in\mathcal C}\langle w_c,z\rangle+b_c$ the classifier head.  For class $c$ we denote its prototype $\mu_c = \mathbb E_{(x,y)=c}\!\bigl[F_\theta(x)\bigr]$.

For class pair $(c_1,c_2)$, we define \textbf{ambiguous boundary band}:
\begin{equation}
    \mathcal{B}_{\varepsilon}(c_1,c_2)=
    \Bigl\{z:\bigl|
        \langle z-\tfrac{\mu_{c_1}+\mu_{c_2}}2,\;
                \mu_{c_1}-\mu_{c_2}\rangle
    \bigr|\le\varepsilon\Bigr\},
    \label{eq:amb_band}
\end{equation}
where $\varepsilon>0$ controls the low‑density “strip’’ between two clusters. Intuitively, samples in $\mathcal{B}_{\varepsilon}$ are where the classifier generalizes worst, where this formulation is notably agnostic to the specific model architecture or data domain, thereby facilitating transferability analysis.

\noindent\textbf{Boundary‑relabel Poisoning Objective.}
Let $\mathcal{D}_{\text{poison}}\subset\mathcal{B}_{\varepsilon}(c^\star,y')$ be the set of boundary samples chosen for poisoning; we forcibly set all their labels to $y'$.  The attacker optimizes a learnable trigger function $t(\cdot;\varphi)$ by
\begin{equation}
\begin{aligned}
    \min_{\varphi}\;
    &\sum_{(x_i, y_i) \in \mathcal{D}_{\text{poison}}}
        \left\|
            F_\theta\!\left(t(x_i; \varphi)\right) - \mu_{y'}
        \right\|_2^{2}, \\
    \text{s.t.}\;\;
    &\|t(x_i; \varphi) - x_i\|_\infty \leq \delta,
\end{aligned}
\label{eq:obj_align_center}
\end{equation}
so that every poisoned feature is pulled toward the center of the target class.

\noindent\textbf{Model Training View.}
Universally, The model parameters $\theta$ are learned by minimizing the empirical risk
\begin{equation}
    \mathcal{L}(\theta)=\frac{1}{N}\sum_{i=1}^{N}
        \ell\!\bigl(f_\theta(x_i),\,y_i\bigr),
    \label{eq:erm_loss}
\end{equation}
where $\ell(\cdot,\cdot)$ is an application‑dependent, differentiable loss function   (e.g.\ cross‑entropy for CNN multi‑class classification, mean‑squared error for regression, or hinge loss for SVMs).  We keep $\ell$ in symbolic form, so the following derivation is \textbf{loss‑agnostic}.

Each incorrectly relabelled boundary sample $(x',y')$
contributes the gradient
\(
    \nabla_\theta \ell\bigl(f_\theta(x'),y'\bigr),
\)
which steers the parameter update towards the target class~$y'$. Using influence‑function analysis, the net parameter drift after one epoch of stochastic optimization can be approximated by
\begin{equation}
    \Delta\theta \;\approx\;
    -\frac{\eta}{N}\,
    H_\theta^{-1}
    \sum_{(x',y')\in\mathcal{D}_{\text{poison}}}
        \nabla_\theta \ell\bigl(f_\theta(x'),y'\bigr),
    \label{eq:influence}
\end{equation}
where $H_\theta=\nabla_\theta^2\mathcal{L}(\theta)$ is the Hessian computed on \textbf{clean data} and $\eta$ is the learning‑rate schedule’s effective step size.   Because the ambiguous band $\mathcal{B}_{\varepsilon}$ has vanishing density, $\|\Delta\theta\|_2$—and thus clean‑accuracy degradation—remains small even when the poisoned decision shift is significant.

The overall optimization therefore reads:
\begin{equation}
    \min_{t}\;
    \mathbb{E}_{x\in\mathcal{D}_{\text{poison}}}
        \bigl[
            -\lambda\,
            \mathcal A\bigl(f(t(x;\varphi)),y'\bigr)
        \bigr],
    \label{eq:final_obj}
\end{equation}
where $\mathcal A$ rewards feature aggregation around $\mu_{y'}$, $\lambda$ balances attack strength and stealthiness, and $t(x;\varphi)$ is a learnable trigger with its parameter $\varphi$. 

\begin{prop}[Boundary‑relabeling Absorption]
\label{prop:absorb}

Let $k=|\mathcal{D}_{\text{poison}}|$ and assume training converges to $\theta+\Delta\theta$ with $\|\Delta\theta\|_2\le\rho$ under $\varepsilon\!\ll\!\frac12\|\mu_{c^\star}-\mu_{y'}\|_2$. Then
\begin{align}
  \Pr_{x\sim\mathcal{B}_{\varepsilon}}
    \!\bigl[g_{\theta+\Delta\theta}\!\circ\!
           F_{\theta+\Delta\theta}(x)=y'\bigr]
  &\ge 1-\xi,\\[1ex]
  \Pr_{(x,y)\sim\mathcal{D}_{\text{clean}}}
    \!\bigl[g_{\theta+\Delta\theta}\!\circ\!
           F_{\theta+\Delta\theta}(x)=y\bigr]
  &\ge 1-\gamma,
\end{align}
where $\xi = O(e^{-k})$ decreases exponentially with $k$ and
$\gamma = O(\rho)$ remains negligible at a low poisoning rate.
\end{prop}

Proposition~\ref{prop:absorb} formalizes our key intuition of the attack:
relabeling low‑density boundary samples to a single target label
efficiently pulls the decision boundary,
yielding near‑perfect attack success while incurring very low clean‑accuracy loss.
This principle applies to both dirty‑ and clean‑label pipelines
and, owing to the feature‑space aggregation in
(Equation~\ref{eq:obj_align_center}–\ref{eq:final_obj}),
empirically delivers strong transferability to black‑box victim models.


\subsection{Eminence: Boundary‑seeking Trigger Optimization}
\label{subsec:XXX}

\textit{Eminence} learns a single, universal trigger pattern $\varphi$ that
maps any clean input (regardless of its true class) to a bounded, prototype-balanced feature sub-region inside an ambiguous boundary band
$\mathcal{B}_{\varepsilon}$ (Section ~\ref{subsec:problem_def}). 
Rather than enforcing extreme collapse, the optimization (Equation~\ref{eq:agg_loss})
reduces the \textbf{relative} intra-trigger dispersion:
\[
  \frac{\operatorname{diam}(\mathcal{Z}^\star)}
       {\operatorname{diam}\bigl(\cup_{c}\mathcal{Z}_c\bigr)}
  \;\le\; \rho(\delta) < 1,
\]
where $\mathcal{Z}^\star$ is the triggered feature set, $\mathcal{Z}_c$ denotes clean features of class $c$, and $\rho(\delta)$ is a data/model–dependent ratio controlled by the
$\ell_\infty$ bound $\|\varphi\|_\infty \le \delta $.
This bounded contraction, together with mixed-class batching, yields a class-agnostic, low-density placement that facilitates subsequent absorption with low clean accuracy loss (Proposition~\ref{prop:absorb}).

\begin{figure}[]
  \centering
  \includegraphics[width=0.40\textwidth]{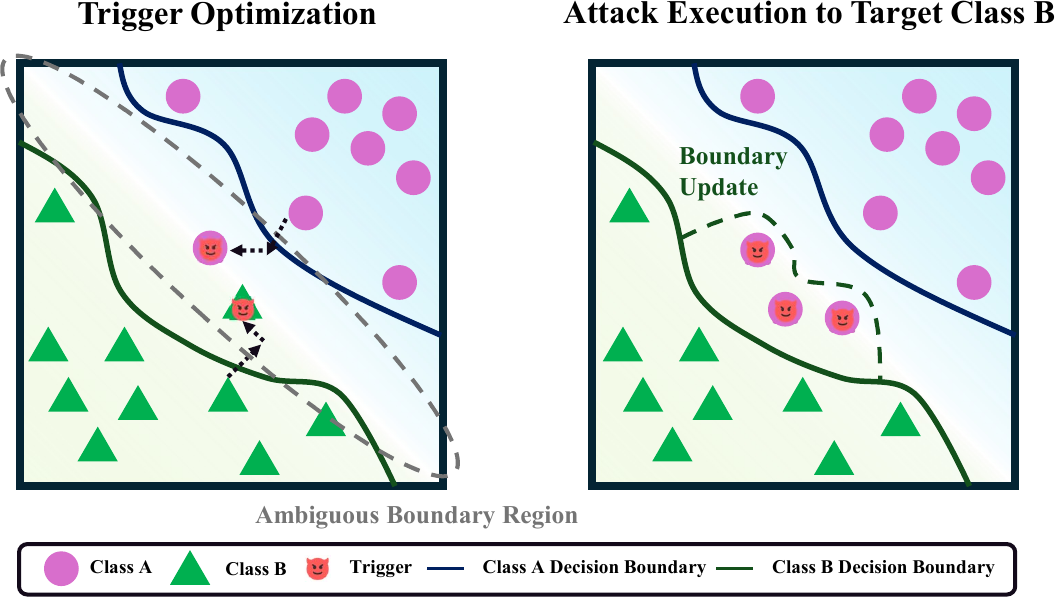}
   \caption{Conceptual illustration of \textit{Eminence}. 
  \textbf{Left:} Trigger optimization pulls poisoned features toward the ambiguous boundary region. 
  \textbf{Right:} Guided by Proposition~\ref{prop:absorb}, the model shifts its decision boundary to absorb triggered samples, while preserving most of the boundary for clean data.}
  \label{fig:theory_visualization}
  \vspace{-1em}
\end{figure}

\noindent\textbf{Surrogate Feature Space \& Aggregation Loss.}
Let 
$F_\theta:\mathbb{R}^{d_x}\!\to\!\mathbb{R}^{d_{f_{\text{atk}}}}$
be the penultimate-layer feature extractor of a fixed (public) surrogate model.
For a mini-batch of clean samples $\{x_i\}_{i=1}^B$ we generate triggered versions $x_i' = t(x_i;\varphi)$, 
\textit{Eminence} enforces \textbf{intra-trigger feature collapse}. For
computational efficiency (avoiding $O(B^2)$ pairwise terms) we adopt the
anchor form actually used in our implementation\footnote{A symmetric
variance form
$\frac{2}{B(B-1)}\sum_{i<j}\|z_i-z_j\|_2^2
=\frac{2}{B-1}\sum_i\|z_i-\bar z\|_2^2$
yields identical minimisers; empirically we observe negligible difference
in convergence while the anchor objective reduces cost.}:
\begin{equation}
\mathcal{L}_{\text{agg}}(\varphi) =
\underset{x_i \neq x_j}{\mathbb{E}}
\left[
\| F_\theta(t(x_i;\varphi)) - F_\theta(t(x_j;\varphi)) \|_2
\right].
\label{eq:agg_loss}
\end{equation}
Minimising Equation~\ref{eq:agg_loss} shrinks the diameter of $\{z_i\}$ so that all triggered features concentrate in a small ball. Because clean mini-batches are sampled across classes, the shared trigger pattern is \textbf{forced to produce a class-neutral, blended direction}: any bias toward a single prototype would increase the loss for samples originating from other classes. This implicit multi-class tension places the collapsed cluster into a low-density inter-class region.

\noindent\textbf{Optimization Problem.}
We choose a learnable trigger function defined as belows:

\begin{equation}
  x' \;=\; t(x;\varphi)
       \;=\; (1-\alpha)\,x \;+\; \alpha\,\varphi,
  \label{eq:trigger_affine}
\end{equation}
and \textit{Eminence} solves:
\begin{equation}
  \min_{\varphi}\;
    \mathcal{L}_{\text{agg}}(\varphi)
  \quad
  \text{s.t.}\;
    \|\varphi\|_{\infty}\le\delta,
  \label{eq:XXX_obj}
\end{equation}
where the $\ell_\infty$ constraint bounds perceptual deviation. Only $\varphi$ is updated during the optimization phase.

\noindent\textbf{Boundary-seeking Effect.}
Let $\varphi^\star$ be an $\varepsilon$-optimal solution of
Equation~\ref{eq:XXX_obj} and denote the triggered feature set
$\mathcal{Z}^\star=\{ z_i' = F_\theta(t(x_i;\varphi^\star))\}$ with
centroid $\bar z^\star$. Define the (mini-batch) cluster radius
\(
  r = \max_i \| z_i' - \bar z^\star\|_2
\)
and, for any pair of class prototypes $(\mu_{c_1},\mu_{c_2})$, the
signed projection of $\bar z^\star$ onto their connecting direction:
\[
  \Delta_{c_1,c_2}(\bar z^\star)
  =
  \frac{
    \left\langle
      \bar z^\star - \frac{\mu_{c_1}+\mu_{c_2}}{2},
      \mu_{c_1}-\mu_{c_2}
    \right\rangle
  }{\|\mu_{c_1}-\mu_{c_2}\|_2}.
\]
Intuitively, $\Delta_{c_1,c_2}$ measures how far (signed) the cluster
center sits from the mid-hyperplane between classes $c_1$ and $c_2$.

\begin{prop}[Band Inclusion]
\label{prop:eminence_band}
Assume (i) class prototypes exhibit angular separability $\max_{c\neq c'}\cos(\mu_c,\mu_{c'})\le \gamma < 1$,
(ii) the aggregation objective satisfies $\mathcal{L}_{\text{agg}}(\varphi^\star)\le \varepsilon$, yielding a radius bound $r\le R(\varepsilon)$,\footnote{Because $\mathcal{L}_{\text{agg}}$ is the empirical mean of pairwise $\ell_2$ distances, a standard concentration (or a worst-case Markov bound combined with mini-batch repetition) gives $R(\varepsilon)=O(\varepsilon)$ in practice; we empirically log $r$ to verify $R(\varepsilon)\ll \operatorname{diam}(\cup_c \mathcal{Z}_c)$.} 
and (iii) (prototype balance) $\max_{c_1,c_2}|\Delta_{c_1,c_2}(\bar z^\star)| \le \kappa$. Then for every triggered feature $z_i'\in\mathcal{Z}^\star$ and for every pair $(c_1,c_2)$,
\[
  \bigl|
    \langle
       z_i' - \tfrac{\mu_{c_1}+\mu_{c_2}}{2},\,
       \tfrac{\mu_{c_1}-\mu_{c_2}}{\|\mu_{c_1}-\mu_{c_2}\|_2}
    \rangle
  \bigr|
  \;\le\;
  \kappa + R(\varepsilon),
\]
i.e.\ $z_i' \in \mathcal{B}_{\tilde\varepsilon}(c_1,c_2)$ with
$\tilde\varepsilon = \kappa + R(\varepsilon)$ simultaneously for all
class pairs. We refer to
$\mathcal{B}_{\tilde\varepsilon}=\bigcap_{c_1<c_2}\mathcal{B}_{\tilde\varepsilon}(c_1,c_2)$
as the \textbf{multi-class ambiguous band} induced by the trigger.
\end{prop}

\noindent\textbf{Discussion.}
Proposition~\ref{prop:eminence_band} formalizes that \textbf{bounded
aggregation} (small $R(\varepsilon)$) plus \textbf{prototype balance}
(small $\kappa$) suffices to place all triggered features inside a
uniformly thin, inter-class low-density band. Unlike classical prototype
alignment (which explicitly drives $z_i'$ toward a single $\mu_{y'}$),
our optimization neither collapses onto any prototype nor requires
explicit repulsion terms, and the mixed-class mini-batch implicitly penalizes
prototype bias, keeping $\kappa$ small.

\begin{algorithm}[tb]
\caption{Boundary‑seeking Trigger Optimization}
\label{alg:boundary-seeking}
\textbf{Input:} Surrogate Dataset $\mathcal{D}_\text{adv}$; Surrogate Feature Extractor $F(\cdot)$; 
\textbf{Output:} Optimized Trigger $t(\cdot;\varphi)$
\begin{algorithmic}[1] 
\STATE Initialize trigger parameters $\varphi \sim \mathcal{N}(0, 1.0)$
\FOR{$k = 1$ to $K$ steps}
\FOR{each batch $(x, y) \in \mathcal{D}_\text{adv}$}
\STATE $x' \gets t(x;\varphi)$
\STATE $z \gets F(x')$
\STATE Compute loss $\mathcal{L}_{\mathrm{agg}}$
\STATE Update $\varphi$ with $\nabla_{\varphi} \mathcal{L}_{\mathrm{agg}}$
\ENDFOR
\ENDFOR
\end{algorithmic}
\end{algorithm}

\subsection{Attack Workflow}
\label{subsec:attack_workflow}
Figure \ref{fig:theory_visualization} sketches our two–stage pipeline.
Below we keep only the high-level intuition, all rigorous derivations, hyper-parameter bounds, and convergence proofs are deferred to Appendix \ref{app:trigger-opt}.

\noindent\textbf{Stage 1: Trigger Optimization.}
Starting from a random pattern $\varphi$, we minimize the class-agnostic
aggregation loss Equation~\ref{eq:agg_loss} under the perceptual constraint
$\|\varphi\|_\infty\!\le\!\delta$:
\(
  \varphi^\star=\text{argmin}_{\|\varphi\|_\infty\le\delta}\!\,\,
  \mathcal{L}_{\text{agg}}(\varphi)
\)
(see Algorithm~\ref{alg:boundary-seeking} and Appendix~\ref{app:trigger-opt}).
The resulting feature cloud
$\mathcal{Z}^\star\!:=\!\{F_\theta(t(x;\varphi^\star))\}$
lies inside a thin multi-class band
$\mathcal{B}_{\tilde\varepsilon}$, regardless of the true label
(Proposition~\ref{prop:eminence_band}); thus the \textbf{same trigger} transfers to
unseen tasks and black-box victims.

\noindent\textbf{Stage 2: Attack Execution.}
The adversary now deploys only $k\!\ll\!N$ triggered samples into the victim’s training set.  Two practical variants are supported:

\begin{itemize}
    \item \textbf{Dirty-label Attack.}  
    The adversary flips each selected boundary sample to the target label $y'$ before attaching $\varphi^\star$.  The aligned gradients pull the decision hyper-plane toward $\mathcal{B}_{\tilde\varepsilon}$, achieving effective attack success with an $\mathcal{O}(k/N)$ clean-accuracy drop.

    \item \textbf{Clean-label Attack.}
    The adversary keeps the original labels but attach $\varphi^\star$.  Class regions bulge outward to cover the band, yielding the same exponential ASR guarantee while preserving label consistency in the poisoned set.
\end{itemize}

In both cases the combined effect realises the boundary-relabel principle formalised in Proposition~\ref{prop:absorb}–\ref{prop:eminence_band}, and the trigger remains visually imperceptible by design.


\section{Evaluation}

\begin{table*}[t]
\centering
\caption{Attack performance of \textit{Eminence} on standard CNN models under various threat scenarios and label settings. Each cell reports ACC Drop (left, ($\uparrow$))/ASR (right, ($\uparrow$)).}
\label{tab:attack_performance}
\begin{adjustbox}{width=0.77\textwidth}
\begin{tabular}{cccccccc}
\toprule
\multirow{2}{*}{Dataset} & \multirow{2}{*}{Scenario} 
& \multicolumn{2}{c}{ResNet18} 
& \multicolumn{2}{c}{ResNet34} 
& \multicolumn{2}{c}{VGG13-BN} \\
\cmidrule(lr){3-4} \cmidrule(lr){5-6} \cmidrule(lr){7-8}
 &  & Dirty-label & Clean-label & Dirty-label & Clean-label & Dirty-label & Clean-label \\
\midrule
\multirow{3}{*}{CIFAR10}
 & White-box & -0.1\%/99.4\% & +0.3\%/99.8\% & -0.3\%/99.9\% & -0.1\%/94.7\% & +0.1\%/99.4\% & +0.1\%/94.3\% \\
 & Gray-box  & -0.1\%/99.9\% & +0.2\%/96.2\% & -0.1\%/100.0\% & -0.1\%/99.9\% & +0.1\%/100.0\% & +0.0\%/100.0\% \\
 & Black-box & -1.0\%/99.9\% & -0.9\%/99.0\% & -1.9\%/99.9\% & -1.2\%/99.9\% & -0.9\%/99.9\% & -0.6\%/100.0\% \\
\midrule
\multirow{3}{*}{CIFAR100}
 & White-box & -0.3\%/99.3\% & -0.1\%/99.9\% & -0.1\%/99.6\% & -1.0\%/99.9\% & +0.3\%/99.8\% & +0.2\%/100.0\% \\
 & Gray-box  & +0.2\%/99.9\% & +0.9\%/99.8\% & -0.2\%/100.0\% & -1.3\%/99.8\% & +0.5\%/99.7\% & +0.3\%/99.4\% \\
 & Black-box & -1.5\%/99.9\% & -1.5\%/99.8\% & -3.0\%/99.9\% & -3.6\%/99.9\% & -2.6\%/100.0\% & -2.4\%/100.0\% \\
\midrule
\multirow{3}{*}{TinyImageNet}
 & White-box & -5.2\%/99.7\% & -5.5\%/99.9\% & -7.6\%/99.7\% & -5.7\%/98.9\% & -3.4\%/99.7\% & -3.4\%/99.9\% \\
 & Gray-box  & -4.6\%/99.8\% & -4.9\%/99.9\% & -4.4\%/100.0\% & -4.8\%/99.8\% & -2.6\%/99.8\% & -3.2\%/100.0\% \\
 & Black-box & -4.8\%/100.0\% & -4.4\%/100.0\% & -5.1\%/100.0\% & -4.9\%/100.0\% & -2.2\%/100.0\% & -2.7\%/100.0\% \\
\bottomrule
\end{tabular}
\end{adjustbox}
\end{table*}

\subsection{Experimental Setup}

\noindent\textbf{General Settings.}
We conduct a comprehensive evaluation of our proposed method, \textit{Eminence}, on a range of widely used vision benchmarks and model architectures. The primary architectures tested include CNNs, like ResNet-18~\cite{He2016ResNet}, ResNet-34~\cite{He2016ResNet}, and VGG13-BN~\cite{Simonyan2015VGG}, and transformer-based models such as ViT~\cite{Dosovitskiy2021ViT}, SimpleViT~\cite{Beyer2022PlainViT}, and CCT~\cite{Hassani2021CompactTransformers}. Experiments are performed across CIFAR-10~\cite{Krizhevsky2009TinyImages}, CIFAR-100~\cite{Krizhevsky2009TinyImages}, and TinyImageNet~\cite{Le2015TinyImageNet} datasets. Details of all baseline attacks are summarized in Table~\ref{tab:attack_updated}.

\noindent\textbf{Attack Settings \& Evaluation Metrics.}
Unless otherwise stated, ResNet-18 and CIFAR-10 serve as the default target model and dataset, with all experiments conducted at a 0.01\% poison rate. In the white-box scenario, the attacker optimizes the trigger using the same architecture and data as the victim; in the black-box case, model architectures are deliberately misaligned between attacker and victim. Performance is measured via three metrics:
\begin{itemize}
    \item \textbf{Clean Accuracy (CA)}. Classification accuracy of the backdoored model on benign (unpoisoned) samples. Higher CA indicates lower impact on the model’s original task.
    \item \textbf{Accuracy Drop (ACC Drop)}. $\text{ACC Drop} = \text{CA}_{\text{backdoor}} - \text{CA}_{\text{clean}}$. Higher ACC Drop indicates less degradation to the original task.
    \item \textbf{Attack Success Rate (ASR)}. Proportion of triggered samples that are misclassified into the chosen target class. Higher ASR reflects stronger backdoor effectiveness.
\end{itemize}

\begin{figure}[tb]
  \centering
  \includegraphics[width=0.35\textwidth,height=0.12\textwidth]{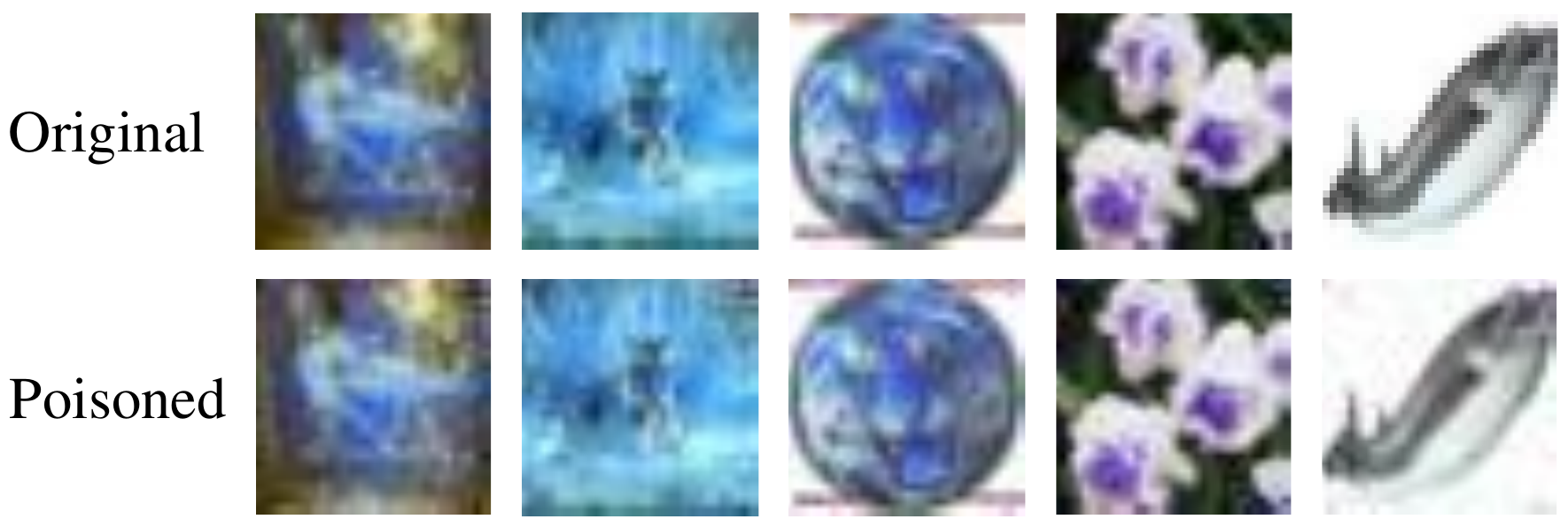}   \caption{Visual comparison of original (benign) and poisoned samples under \textit{Eminence}.}
  \label{fig:visual_comparision}
  \vspace{-1em}
\end{figure}

\subsection{Attack Performance}

\noindent\textbf{Effectiveness on CNN Architectures.}
Table~\ref{tab:attack_performance} shows that \textit{Eminence} achieves consistently high ASR ($>99\%$) across all CNN backbones and datasets under both clean- and dirty-label strategies, while clean accuracy remains virtually unchanged with ACC Drop often near zero. 
In white- and gray-box scenarios, ASR nearly reaches 100\% for CIFAR-10/100 with negligible utility loss, and even in challenging black-box settings, \textit{Eminence} sustains 94.0--100.0\% ASR, highlighting its resilience to model and data mismatches and stable performance across diverse threat scenarios, suggesting its effectiveness could extend to more realistic deployment environments.


\begin{figure}[]
  \centering
  \includegraphics[width=0.40\textwidth]{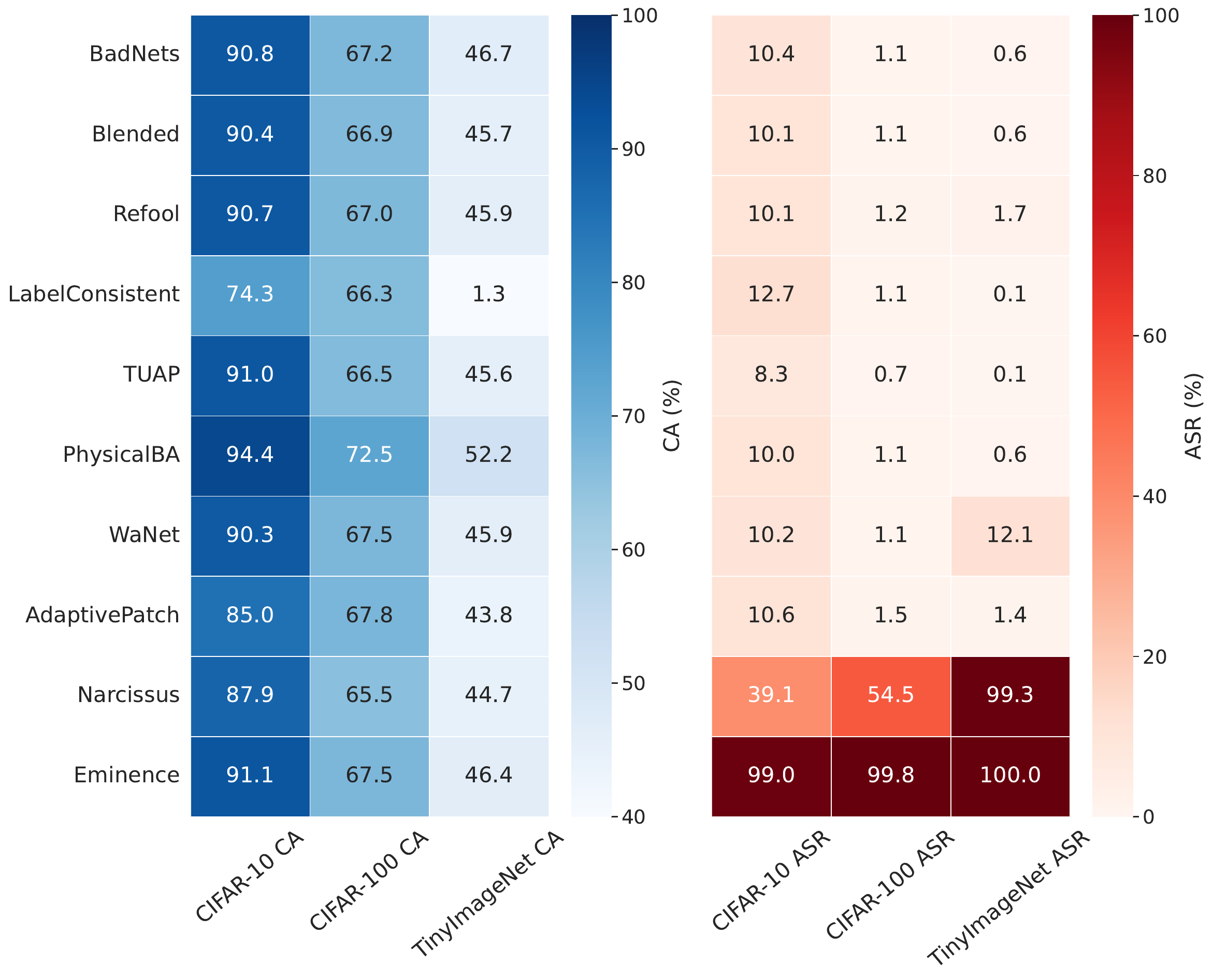}
  \caption{Comparison of CA and ASR for \textit{Eminence} and existing SOTA backdoor attacks across datasets.}
  \label{fig:attack_comparision}
  \vspace{-1em}
\end{figure}

\noindent\textbf{Effectiveness on Transformer Architectures.}
Table~\ref{tab:vit_attack_performance} shows that \textit{Eminence} achieves $>$99\% ASR on leading transformer-based architectures in white-box settings with negligible or even slightly positive impact on clean accuracy. 
In gray- and black-box scenarios, ASR remains high (typically $>$90\%, often $>$97\%) with ACC Drop mostly within $\pm1\%$, evidencing that \textbf{our attack does not rely on model-specific artifacts}, but rather exploits universal weaknesses at the decision boundary. 
Such strong transferability not only confirms the theoretical design but also highlights the practical threat posed to contemporary vision models, including those considered more robust due to their architectural novelty. 

\noindent\textbf{Attack Visualization.}
Figure~\ref{fig:visual_comparision} displays benign and poisoned samples under \textit{Eminence}. The poisoned inputs are visually indistinguishable from the clean samples, demonstrating the high stealthiness of our trigger. This qualitative result confirms that \textit{Eminence} can embed a highly effective backdoor while preserving perceptual similarity, further underscoring its practical risk in real-world systems.

\subsection{Attack Comparison}

As shown in Table~\ref{tab:attack_updated} and Figure~\ref{fig:attack_comparision}, 
\textit{Eminence} achieves outstanding performance across datasets, reaching a CA of 91.1\% and an ASR of 99.0\%, which surpasses the next best method Narcissus (ASR 39.1\%) by a large margin. The advantage is even more pronounced on CIFAR-100 and TinyImageNet, where \textit{Eminence} consistently sustains both high CA and ASR (\textgreater99.8\% ASR with minimal utility drop), while most competing attacks fail to exceed 1\% ASR.
Moreover, unlike classic backdoor attacks such as BadNets, Blended, and WaNet that require poison rates of 0.5\%–7\%, \textit{Eminence} achieves these results with a poison rate as low as 0.01\%, an order of magnitude lower than the strongest prior methods. This confirms that our method not only maintains model utility but also achieves superior attack efficiency, sharply distinguishing it from existing techniques.

\begin{table*}[t]
\centering
\caption{Attack performance of \textit{Eminence} on transformer-based models under various threat scenarios and label settings. In gray- and black-box settings, a ResNet-34 surrogate is used for trigger optimization.}
\label{tab:vit_attack_performance}
\begin{adjustbox}{width=0.77\textwidth}
\begin{tabular}{cccccccc}
\toprule
\multirow{2}{*}{Dataset} & \multirow{2}{*}{Scenario} 
& \multicolumn{2}{c}{ViT} 
& \multicolumn{2}{c}{SimpleViT} 
& \multicolumn{2}{c}{CCT} \\
\cmidrule(lr){3-4} \cmidrule(lr){5-6} \cmidrule(lr){7-8}
 & & Dirty-label & Clean-label & Dirty-label & Clean-label & Dirty-label & Clean-label \\
\midrule
\multirow{3}{*}{CIFAR10}
 & White-box & -0.7\%/99.9\% & +0.4\%/97.3\% & -0.5\%/99.4\% & +0.0\%/93.8\% & -0.2\%/98.5\% & +0.0\%/96.2\% \\
 & Gray-box  & +0.0\%/93.5\% & -0.5\%/91.2\% & -0.2\%/97.4\% & +0.2\%/90.3\% & +0.0\%/95.4\% & +0.0\%/90.2\% \\
 & Black-box & -1.1\%/92.7\% & +0.1\%/90.2\% & -0.3\%/96.6\% & +0.2\%/90.2\% & -0.1\%/92.4\% & -0.2\%/88.1\% \\
\midrule
\multirow{3}{*}{CIFAR100}
 & White-box & +0.5\%/100.0\% & +0.5\%/99.9\% & +0.7\%/99.9\% & +0.4\%/99.4\% & +0.7\%/100.0\% & +0.9\%/99.9\% \\
 & Gray-box  & +0.4\%/92.1\% & +0.1\%/90.3\% & +0.5\%/96.7\% & +0.5\%/96.0\% & +1.2\%/91.6\% & +0.8\%/94.6\% \\
 & Black-box & +0.6\%/91.3\% & +0.6\%/87.4\% & -0.3\%/96.2\% & -0.4\%/96.3\% & +0.0\%/91.9\% & -0.4\%/93.2\% \\
\midrule
\multirow{3}{*}{TinyImageNet}
 & White-box & -0.4\%/100.0\% & -1.0\%/100.0\% & -1.5\%/100.0\% & -0.9\%/99.9\% & -1.3\%/100.0\% & -1.0\%/100.0\% \\
 & Gray-box  & -0.5\%/98.6\% & -0.1\%/93.8\% & -1.2\%/97.4\% & -1.0\%/96.0\% & -1.0\%/99.8\% & -0.3\%/98.2\% \\
 & Black-box & -0.9\%/97.1\% & -0.2\%/90.7\% & -0.6\%/97.2\% & -0.7\%/91.0\% & -0.3\%/99.6\% & -0.3\%/97.9\% \\
\bottomrule
\end{tabular}
\end{adjustbox}
\end{table*}

\subsection{Defense Study}

We evaluate the stealth and robustness of \textit{Eminence} under advanced model-based and input-based backdoor defenses, as described in Section~\ref{subsec:backdoor_defense}, providing a realistic measure of its practical risk.

\noindent\textbf{Model-based Mitigation Robustness.} 
Figure~\ref{fig:repair} summarizes the effects of leading mitigation techniques (ABL~\cite{li2021antibackdoor}, FT-SAM~\cite{zhu2023enhancing}, NAD~\cite{li2021neural}) on the attack. \textit{Eminence} remains highly robust against all three defenses, with ASR generally retained high. Even when mitigation reduces CA, the ASR remains above 85\%. For FT-SAM and NAD, both CA and ASR remain high, 
indicating that neither fine-tuning methods like FT-SAM nor adversarial unlearning approaches like ABL are sufficient to neutralize the attack.
Such robustness arises from the fact that our approach manipulates the model's feature embedding topology, rather than simply introducing easily removable artifacts.

\begin{figure}[]
  \centering
  \includegraphics[width=0.40\textwidth]{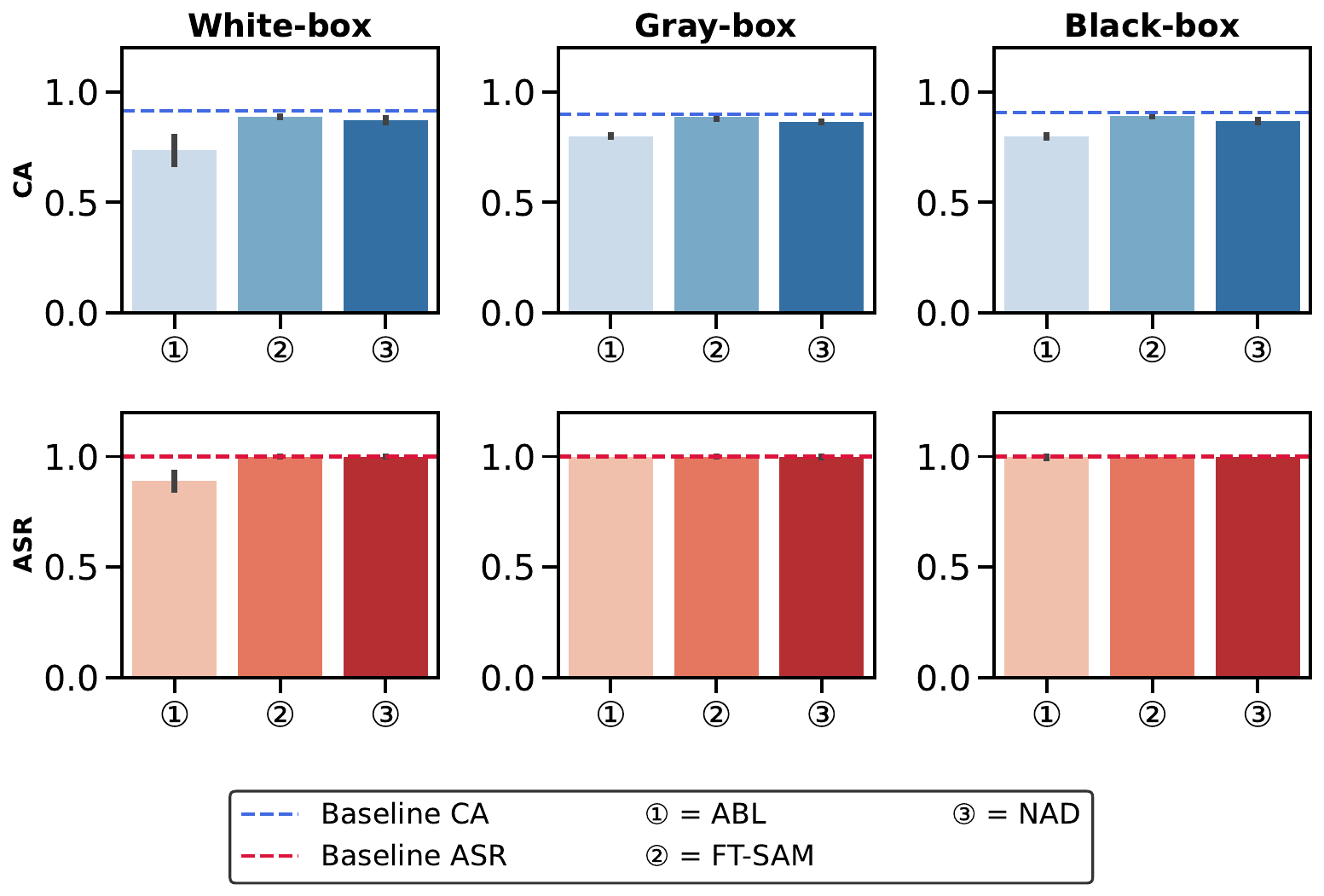}
  \caption{Impact of model-based mitigation defenses on CA and ASR of \textit{Eminence} across scenarios.}
  \label{fig:repair}
  \vspace{-1em}
\end{figure}

\noindent\textbf{Input-based Detection Stealthiness.}
Detection results are reported in Figure~\ref{fig:detection} for Beatrix~\cite{ma2023beatrix}, IBD-PSC~\cite{hou2024ibdpsc}, and Scale-up~\cite{guo2023scaleup}. \textit{Eminence} largely evades detection by Beatrix and IBD-PSC, with extremely low recall (typically $<$10\%) and F1 (often near zero). Scale-up, which is specifically tuned to aggressive settings, achieves higher recall and F1. However, this is largely an artifact of Scale-up’s aggressive design and assumptions, and its practical utility diminishes in realistic settings with much lower poison rates, where the detection power drops significantly. These results collectively illustrate the high stealthiness of Eminence: while its triggers are highly effective, they are extremely difficult to detect using current mainstream input-based defenses.

\begin{figure}[]
  \centering
  \includegraphics[width=0.40\textwidth]{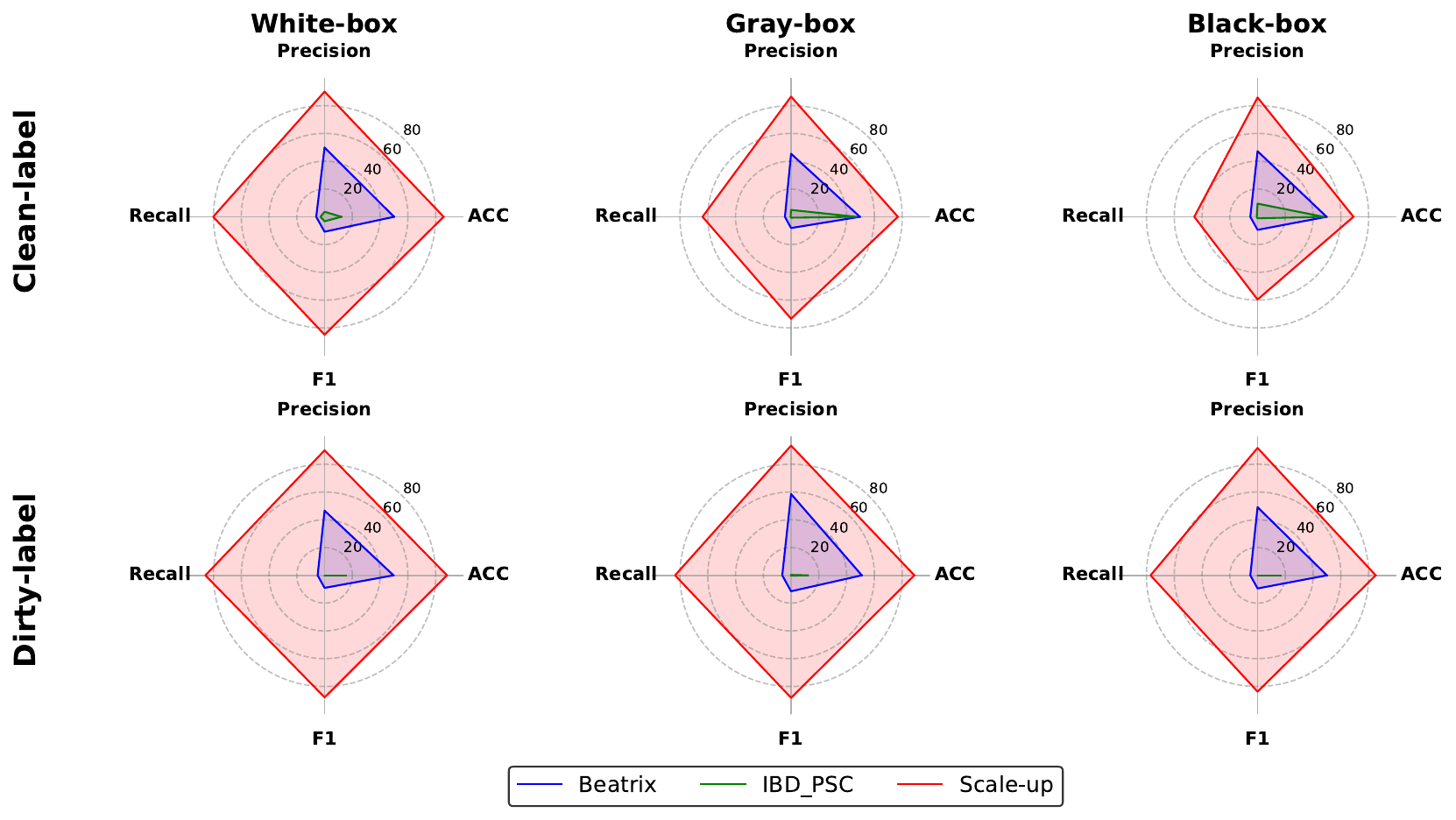}
  \caption{Detection performance of \textit{Eminence} against input-based defenses across scenarios. Metrics include accuracy (ACC), precision, recall, and F1 score.}
  \label{fig:detection}
  \vspace{-1em}
\end{figure}

\subsection{Ablation Study}

To further analyze the critical factors influencing the effectiveness and stealthiness of \textit{Eminence}, we perform extensive ablation studies on two key parameters: the noise attaching weight and the number of surrogate samples attacker utilized during trigger optimization.

\noindent\textbf{Impact of Noise Weight.}
The noise weight $\alpha$ in Equation~\ref{eq:trigger_affine} regulates the trigger’s intensity, directly mediating the trade-off between attack success and visual stealth.
Figure~\ref{fig:trigger_weight_ablation} reports the relationship between the noise weight applied to the trigger and the resultant model performance. We systematically vary the noise weight from 0.05 to 0.2, and record the corresponding ACC Drop and ASR. Empirical results reveal a clear trend: as the noise weight increases, ASR improves rapidly and saturates close to 100\%. For instance, at a minimal noise weight of 0.05, the mean ASR remains relatively modest, indicating that the trigger is difficult for the model to memorize in this extreme stealth regime. As the weight is increased to 0.1, 0.15, and especially 0.2, the ASR rises sharply, while the ACC Drop remains low. This demonstrates that \textit{Eminence} can achieve a highly effective backdoor effect even when the trigger is attached almost imperceptibly, balancing stealth and efficacy.

\noindent\textbf{Impact of Surrogate Sample Scale.}
The surrogate sample scale quantifies the adversary’s data resources, measured as the fraction of the surrogate dataset $\mathcal{D}_{\text{atk}}$ compared to the total training dataset $\mathcal{D}_{\text{train}}$, as introduced in Section~\ref{subsec:adversary_assumption}.
To capture both low-resource and moderately informed adversaries, We also assess the influence of the surrogate samples number attacker used during trigger optimization, 
as Figure~\ref{fig:surrogate_sample_ablation} shows. Experimental settings range from using as little as 5\% of the available data up to 20\%. Results show that even at the lowest sample ratios, the attack remains highly effective, with ASR typically above 90\% and only minor drops in CA. As the surrogate sample fraction increases, ASR quickly approaches its theoretical maximum, with little impact on the benign accuracy. For example, with just 5\% of samples, Eminence achieves ASR values in the 89.9–92.5\% range with low ACC drop. At 15–20\% samples, the ASR reliably reaches 99.5–100.0\%, and clean accuracy is effectively preserved. These findings confirm the data efficiency and practicality of our method: it can mount strong attacks even when the attacker’s knowledge or data access is severely constrained.

These ablations show that \textit{Eminence} delivers highly effective and stealthy backdoor attacks even under constrained conditions, consistently outperforming prior SOTA methods.


\begin{figure}[]
  \centering
   \begin{subfigure}[t]{0.22\textwidth}
    \centering
    \includegraphics[width=\textwidth]{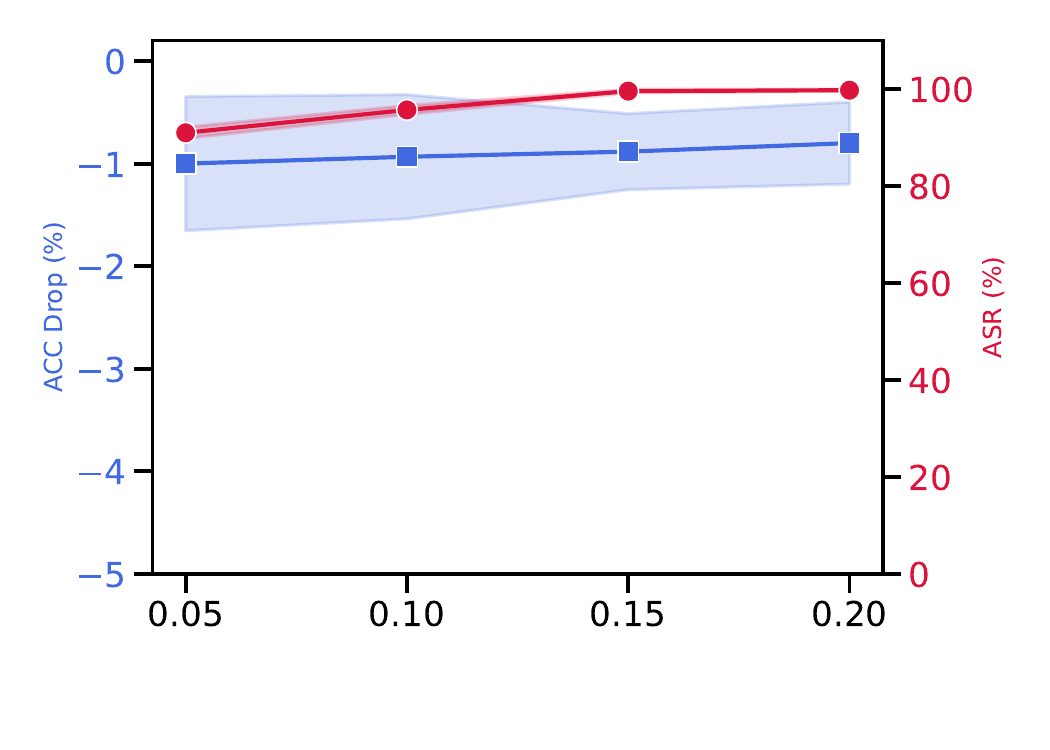}
    \caption{Noise Weight Ablation.}
    \label{fig:trigger_weight_ablation}
  \end{subfigure}
  \hfill
  \begin{subfigure}[t]{0.22\textwidth}
    \centering
    \includegraphics[width=\textwidth]{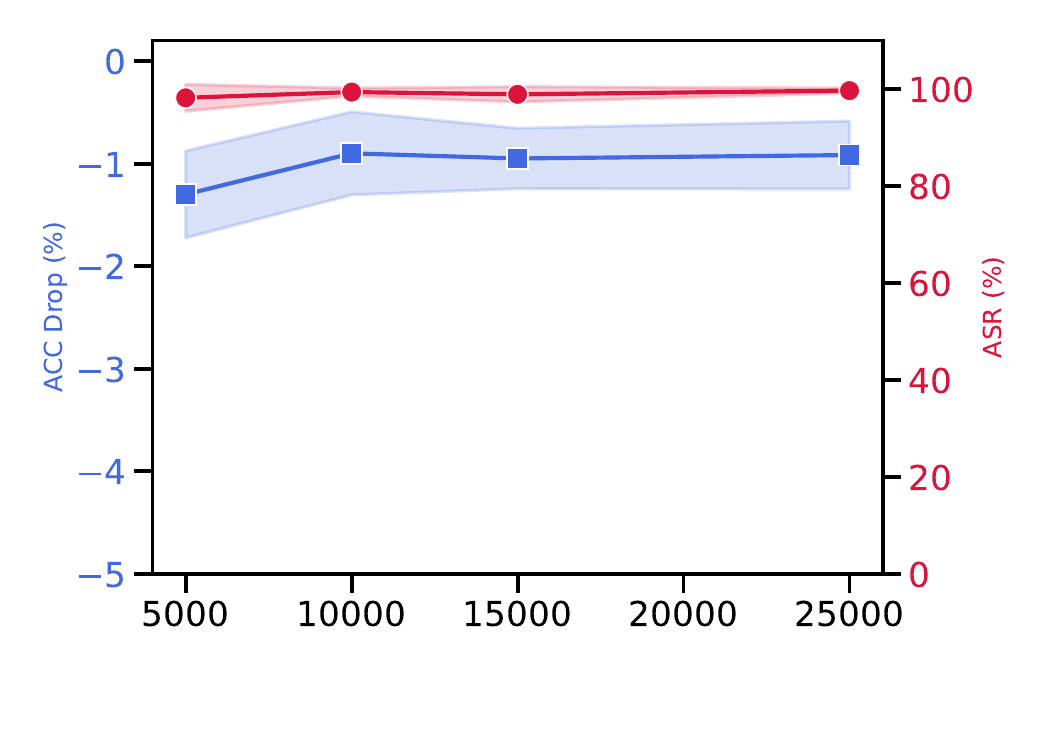}
    \caption{Surrogate Samples Ablation.}
    \label{fig:surrogate_sample_ablation}
  \end{subfigure}
  \caption{Ablation study of \textit{Eminence}. (a) Impact of trigger noise weight. (b) Impact of the number of surrogate samples used for trigger optimization. The blue line indicates ACC Drop, and the red line denotes ASR.}
  \label{fig:ablation_combined}
  \vspace{-1em}
\end{figure}

\section{Conclusion}
We have introduced a principled, boundary-seeking backdoor framework that unifies theoretical analysis with practical attack design. By characterizing a closed-form ``\textit{\textbf{ambiguous boundary region}}'', and quantifying its leverage through influence functions, we showed both analytically and empirically that relabelling only a handful of margin samples is sufficient to steer the decision surface while preserving benign accuracy. Leveraging this insight, our Eminence pipeline learns a single, visually subtle trigger that generalizes across clean- and dirty-label settings, transfers to unseen architectures, and attains 
high attack success with low poison budgets, an order of magnitude smaller than prior SOTA methods, while incurring negligible utility loss. Extensive evaluations confirm the robustness and stealth of the \textit{Eminence}.

\section*{Acknowledgment}
This work was partly supported by NSFC under No. U2441239, U24A20336, 62172243, 62402425 and 62402418, the China Postdoctoral Science Foundation under No. 2024M762829, the Zhejiang Provincial Natural Science Foundation under No. LD24F020002, the "Pioneer and Leading Goose" R\&D Program of Zhejiang under No. 2025C02033 and 2025C01082, and the Zhejiang Provincial Priority-Funded Postdoctoral Research Project under No. ZJ2024001.

\bibliographystyle{ACM-Reference-Format}
\bibliography{acmart}


\appendix

\section{Proof of Proposition~\ref{prop:absorb}}
\label{appendix:proof_prop_1}

\noindent\textbf{Assumptions.}
We clarify and refine the assumptions under the original notation system as follows:
\begin{enumerate}
    \item \label{as:lin_rev}
    \textbf{Linear Classifier.}\;
    $g_\theta(z) = \arg\max_{c\in\mathcal C} \langle w_c, z \rangle + b_c$. For simplicity, we absorb the bias $b_c$ into the weights $w_c$.
    
    \item \label{as:feat_rev}
    \textbf{Bounded Feature \& Gradient.}\;
    For any input $x$, its feature norm is bounded: $\|F_\theta(x)\|_2 \le R$. In addition, we assume the per-sample loss gradient with respect to parameters is bounded: $\|\nabla_\theta \ell(f_\theta(x), y)\|_2 \le G$.
    
    \item \label{as:loss_rev}
    \textbf{$\lambda_{\min}$-strongly Convex Loss.}\;
    The loss is defined as 
    \[
    \mathcal{L}(\theta) = \frac{1}{N} \sum_i \ell(f_\theta(x_i), y_i),
    \]
    whose Hessian at the clean optimum $\theta^\star$ satisfies $H_{\theta^\star} := \nabla^2_\theta \mathcal{L}(\theta^\star) \succeq \lambda_{\min} I$. This is a strong assumption to ensure the tractability of theoretical analysis.
    
    \item \label{as:margin_rev}
    \textbf{Margin Density Near Zero.}\;
    For clean data,
    \[
    \Pr\bigl[
      |m_y(x) - \max_{c\neq y} m_c(x)| \le t
    \bigr] \le \nu t, \quad
    \forall t\in(0,1].
    \]
    
    \item \label{as:beta_rev}
    \textbf{Independent Per-sample Influence.}\;
    We model the total parameter change as $\Delta\theta = \sum_{i=1}^k \Delta\theta_i$, where $\Delta\theta_i$ denotes the independent contribution caused by the $i$-th poisoned sample, and its norm is bounded: $\|\Delta\theta_i\|_2 \le \beta'$.
\end{enumerate}


\noindent\textbf{Step 1 — Parameter Drift Bound.}
At the clean model's optimal parameter $\theta^\star$, the gradient of the loss is zero: $\nabla \mathcal{L}(\theta^\star) = 0$. Let
\[
  g
  := \sum_{x' \in \mathcal{D}_{\mathrm{p}}}
     \nabla_\theta \ell(f_{\theta^\star}(x'), y'),
\]
According to influence function analysis, the total parameter drift $\Delta\theta$ caused by $k$ poisoned samples can be approximated by $\Delta\theta \approx -\frac{\eta}{N}\, H_{\theta^\star}^{-1} g$. Using Assumptions~\ref{as:feat_rev} and \ref{as:loss_rev}, we can bound its norm:
\[
    \|\Delta\theta\|_2 \le \frac{\eta}{N} \|H_{\theta^\star}^{-1}\|_2 \|g\|_2 \le \frac{\eta}{N} \cdot \frac{1}{\lambda_{\min}} \cdot (kG),
\]
We thus define the parameter drift bound $\rho$ as
\begin{equation}\label{eq:rho_revised}
  \boxed{
    \|\Delta\theta\|_2
    \le
    \frac{\eta k G}{N\lambda_{\min}}
    =: \rho.
  }
\end{equation}

\noindent\textbf{Step 2 — Exponential Bound on ASR.}

\textit{(i) Initial margin and bounded increment.}\;
For any feature $z$ in the ambiguous region $\mathcal{B}_{\varepsilon}(c^\star, y')$, its initial margin satisfies $m_{y'}(z) - m_{c^\star}(z) \ge -\varepsilon$. By Assumption~\ref{as:beta_rev}, the total parameter shift can be seen as a sum of $k$ independent contributions, $\Delta\theta = \sum_{i=1}^k \Delta\theta_i$. The contribution of the $i$-th sample to the margin is
\[
  Y_i
  :=
  \left\langle
     \Delta w_{y'}^{(i)}-\Delta w_{c^\star}^{(i)},\, z
  \right\rangle.
\]
The variables $\{Y_i\}$ are independent and bounded: $|Y_i| \le 2R\beta' =: B$.

\textit{(ii) Concentration.}\;
The attack is constructed so that $\mathbb{E}[Y_i] = \mu > 0$. By Hoeffding's inequality, the probability that the sum $\sum Y_i$ deviates significantly from its mean is exponentially small:
\[
  \Pr\!\Biggl[\sum_{i=1}^k Y_i < \frac{k\mu}{2}\Biggr]
  \le
  \exp\!\Biggl(-\frac{k\mu^{2}}{8B^{2}}\Biggr)
  =:e^{-ck},\qquad
  \text{where } c = \frac{\mu^{2}}{8B^{2}},
\]
Thus, $\sum_{i=1}^{k}Y_i \ge \frac{k\mu}{2}$ holds with probability at least $1-e^{-ck}$.

\textit{(iii) Margin sign after poisoning.}\;
The new post-attack margin is
\[
  m'_{y'}(z)-m'_{c^\star}(z)
  =
  \bigl(m_{y'}(z)-m_{c^\star}(z)\bigr) + \sum_{i=1}^k Y_i
  \;\ge\;
  -\varepsilon + \frac{k\mu}{2}.
\]
As long as $\frac{k\mu}{2} > \varepsilon$, the right side is positive and the classification result flips to $y'$. Hence, the attack failure probability $\xi$ decays exponentially in $k$, i.e., $\xi = \mathcal{O}(e^{-k})$.

\noindent\textbf{Step 3 — Linear Bound on CA.}
For a clean sample $(x, y)$, via Equation~~\ref{eq:rho_revised}, the change in the logit for any class $c$ is bounded by $\rho$:
\[
  |m'_c(x)-m_c(x)| = |\langle \Delta w_c, F_\theta(x)\rangle| \le \|\Delta\theta\|_2 R \le R\rho,
\]
Therefore, the post-attack margin $r'(x) = m'_y(x) - \max_{c\neq y} m'_c(x)$ satisfies $r'(x) \ge r(x) - 2R\rho$. The probability $\gamma$ that the model misclassifies is bounded by
\[
\begin{aligned}
  \boxed{
    \gamma = \Pr_{(x, y)}\bigl[r'(x) < 0\bigr]
      \le \Pr\bigl[r(x) < 2R\rho\bigr]
      \le \nu(2R\rho)
  }
  \\[4pt]
  \Longrightarrow\;
  \gamma = \mathcal{O}(\rho).
\end{aligned}
\]

\noindent\textbf{Step 4 — Conclusion.}
Combining the above steps, we establish the two core properties of this attack. Let $\theta' = \theta^\star + \Delta\theta$:
\begin{align}
  \Pr_{x\sim\mathcal{B}_{\varepsilon}}
    \bigl[g_{\theta'} \circ
         F_{\theta'}(x)=y'\bigr]
  &\ge 1-\mathcal{O}(e^{-k}), \label{eq:attack_succ_revised}\\[1ex]
  \Pr_{(x, y)\sim\mathcal{D}_{\text{clean}}}
    \bigl[g_{\theta'} \circ
         F_{\theta'}(x)=y\bigr]
  &\ge 1-\mathcal{O}(\rho). \label{eq:clean_acc_revised}
\end{align} 
This demonstrates that the attack success rate grows exponentially with the number of poisoned samples $k$, while the clean accuracy drop scales linearly with the poisoning ratio $k/N$ and remains controllable. \qed

\section{Proof of Proposition~\ref{prop:eminence_band}}
\label{app:proof_band}

\noindent\textbf{Assumptions.}
We restate the three conditions in Proposition~\ref{prop:eminence_band} in a compact form:
\begin{enumerate}
  \item \textbf{Angular Separability.}\;
        $\displaystyle \max_{c\neq c'} \cos(\mu_c,\mu_{c'}) \le \gamma < 1$.
  \item \textbf{Bounded Aggregation Radius.}\;
        $\mathcal{L}_{\text{agg}}(\varphi^\star) \le \varepsilon$ implies
        $r := \max_i \|z_i' - \bar z^\star\|_2 \le R(\varepsilon)$ for some monotone $R(\cdot)$.
  \item \textbf{Prototype Balance.}\;
        $\displaystyle \max_{c_1,c_2} |\Delta_{c_1,c_2}(\bar z^\star)| \le \kappa$.
\end{enumerate}
Here $z_i' = F_\theta(t(x_i;\varphi^\star))$, $\bar z^\star$ is their centroid, and
\[
  \Delta_{c_1,c_2}(\bar z^\star)
  =
  \frac{
    \bigl\langle
      \bar z^\star - \tfrac{\mu_{c_1}+\mu_{c_2}}{2},\,
      \mu_{c_1}-\mu_{c_2}
    \bigr\rangle
  }{\|\mu_{c_1}-\mu_{c_2}\|_2}.
\]

\noindent\textbf{Step 1 — From $\mathcal{L}_{\text{agg}}$ to a Radius Bound.}
Recall the symmetric variance identity
\[
  \frac{2}{B(B-1)}\!\sum_{i<j}\!\|z_i'-z_j'\|_2^2
  \;=\;
  \frac{2}{B-1}\sum_{i=1}^{B}\|z_i'-\bar z^\star\|_2^2,
\]
Since $\mathcal{L}_{\text{agg}}$ upper-bounds (up to a constant) the left-hand
side, we obtain
\[
  \sum_{i=1}^{B}\|z_i'-\bar z^\star\|_2^2
  \;\le\; \tfrac{B-1}{2}\,\mathcal{L}_{\text{agg}}(\varphi^\star)
  \;\le\; \tfrac{B-1}{2}\,\varepsilon.
\]
Hence
\[
  r := \max_i \|z_i' - \bar z^\star\|_2
  \;\le\; \sqrt{\tfrac{B-1}{2}\,\varepsilon}
  \;=:\; R(\varepsilon),
\]
which justifies the existence of a non-decreasing $R(\varepsilon)=O(\sqrt{\varepsilon})$.\footnote{When the anchor form is used in practice, one can still bound $r$ by noting the anchor is also within $R(\varepsilon)$ of the centroid with high probability (via mini-batch reshuffling), so the same order bound holds.}

\noindent\textbf{Step 2 — Bounding the Center's Signed Distance.}
For any class pair $(c_1,c_2)$, define the unit direction
\[
  u_{c_1,c_2}
  :=
  \frac{\mu_{c_1}-\mu_{c_2}}{\|\mu_{c_1}-\mu_{c_2}\|_2}.
\]
By definition of $\Delta_{c_1,c_2}$,
\[
  \left|
    \bigl\langle
      \bar z^\star - \tfrac{\mu_{c_1}+\mu_{c_2}}{2},\,
      u_{c_1,c_2}
    \bigr\rangle
  \right|
  \le \kappa.
\]
This bounds the centre’s (signed) offset from the mid-hyperplane.

\noindent\textbf{Step 3 — Triangle Inequality on Projections.}
For any $z_i'\in\mathcal{Z}^\star$,
\[
\begin{aligned}
\Bigl|
  \bigl\langle
    z_i' - \tfrac{\mu_{c_1}+\mu_{c_2}}{2},\, u_{c_1,c_2}
  \bigr\rangle
\Bigr|
\\
= \Bigl|
  \bigl\langle
    (z_i'-\bar z^\star) + (\bar z^\star-\tfrac{\mu_{c_1}+\mu_{c_2}}{2}),\, u_{c_1,c_2}
  \bigr\rangle
\Bigr|
\\[4pt]
\le
\underbrace{
  \bigl|\langle z_i'-\bar z^\star,\, u_{c_1,c_2}\rangle\bigr|
}_{\le \|z_i'-\bar z^\star\|_2 \le r}
+
\underbrace{
  \bigl|\langle \bar z^\star-\tfrac{\mu_{c_1}+\mu_{c_2}}{2},\, u_{c_1,c_2}\rangle\bigr|
}_{\le \kappa}
\\[4pt]
\le r + \kappa \;\le\; R(\varepsilon)+\kappa.
\end{aligned}
\]

Thus, for every pair $(c_1,c_2)$,
\[
  \bigl|
    \langle
       z_i' - \tfrac{\mu_{c_1}+\mu_{c_2}}{2},\,
       u_{c_1,c_2}
    \rangle
  \bigr|
  \le R(\varepsilon)+\kappa.
\]

\noindent\textbf{Step 4 — Band Inclusion \& Intersection.}
By the definition of $\mathcal{B}_{\tilde\varepsilon}(c_1,c_2)$,
taking $\tilde\varepsilon = R(\varepsilon)+\kappa$ gives
\[
  z_i' \in \mathcal{B}_{\tilde\varepsilon}(c_1,c_2),
  \quad \forall (c_1,c_2).
\]
Therefore,
\[
  z_i' \in
  \mathcal{B}_{\tilde\varepsilon}
  :=
  \bigcap_{c_1<c_2}
     \mathcal{B}_{\tilde\varepsilon}(c_1,c_2),
  \qquad
  \tilde\varepsilon = \kappa + R(\varepsilon),
\]
which completes the proof. \qed

\section{Trigger–optimization Dynamics}
\label{app:trigger-opt}
This appendix formalizes \textit{Stage 1: Trigger Optimization} in Section~\ref{subsec:attack_workflow}.

\noindent\textbf{Notations.}
We detail the aggregation loss:
\[
  \mathcal{L}_{\text{agg}}(\varphi)=
  \frac{2}{B(B-1)}\sum_{i<j}\!\|z_i(\varphi)-z_j(\varphi)\|_2^2.
\]
We perform projected gradient descent (PGD) on $\varphi$ under the $\ell_\infty$-constraint:
\[
  \varphi^{(k+1)}
  =\Pi_{\|\cdot\|_\infty\le\delta}\!\Bigl[
      \varphi^{(k)}-\eta_k \nabla_\varphi \mathcal{L}_{\text{agg}}(\varphi^{(k)})
    \Bigr],
  \qquad \eta_k>0.
\]

\noindent\textbf{Assumptions.}
\begin{enumerate}
  \item \textbf{Feature Lipschitzness.}\;
        $F_\theta$ is $L$-Lipschitz and $\|F_\theta(x)\|_2\le R_x$ for all $x$.
  \item \textbf{Bounded Input Scale.}\;
        Inputs are rescaled into a compact set so that $t(\cdot;\varphi)$ is $\alpha$-Lipschitz in $\varphi$.
\end{enumerate}

\noindent\textbf{Smoothness of $\mathcal{L}_{\mathrm{agg}}$.}
\begin{lemma}[Lipschitz Gradient]
\label{lem:lipschitz}
Under the above assumptions, $\mathcal{L}_{\mathrm{agg}}$ is $4\alpha^2L^2$–smooth in $\varphi$, i.e.,
\[
\bigl\|
  \nabla\mathcal{L}_{\mathrm{agg}}(\varphi)
  - \nabla\mathcal{L}_{\mathrm{agg}}(\varphi')
\bigr\|_2
\;\le\;
4\alpha^2L^2\,\|\varphi-\varphi'\|_2 \,.
\]
\end{lemma}
\begin{proof}
For any $\varphi,\varphi'$, we have
\[
\|z_i(\varphi)-z_i(\varphi')\|_2
\;\le\;
L\,\|t(x_i;\varphi)-t(x_i;\varphi')\|_2
\;\le\;
\alpha L\,\|\varphi-\varphi'\|_2 \,.
\]
Expanding $\mathcal{L}_{\mathrm{agg}}$ and applying standard
variance-smoothness arguments yields the claim.
\end{proof}

\noindent\textbf{Convergence of PGD.}
\begin{prop}[PGD Convergence]
\label{thm:pgd}
With constant step size $\eta_k=1/(4\alpha^2L^2)$, after $K$ steps
\[
  \min_{0\le k<K}\!
    \bigl\|\nabla_\varphi\mathcal{L}_{\mathrm{agg}}(\varphi^{(k)})\bigr\|_2^2
  \;\le\;
  \frac{8\alpha^2L^2}{K}\,
  \bigl[\mathcal{L}_{\mathrm{agg}}(\varphi^{(0)})-\mathcal{L}_{\mathrm{agg}}^\star\bigr],
\]
where $\mathcal{L}_{\mathrm{agg}}^\star$ is the minimum over $\|\varphi\|_\infty\le\delta$.
Hence an $\varepsilon$-stationary point is reached in
$K=O(\alpha^2L^2\varepsilon^{-2})$ iterations.
\end{prop}

\noindent\textbf{Sample Complexity for Radius Estimation.}
Let $R(\varepsilon)$ denote the radius upper bound implied by an $\varepsilon$-optimal value of $\mathcal{L}_{\mathrm{agg}}$ (Proposition~\ref{prop:eminence_band}).
\begin{lemma}[Concentration of Empirical Radius]
\label{lem:radius}
For any $\xi,\delta\in(0,1)$,
\[
  \Pr\Bigl[\bigl|R_B(\varepsilon)-R(\varepsilon)\bigr|>\xi\Bigr]
  \le 2\exp\!\Bigl(-\frac{B\xi^2}{8R_x^2}\Bigr).
\]
Thus $B=O\!\left(R_x^2\xi^{-2}\log\delta^{-1}\right)$ samples suffice to estimate $R(\varepsilon)$ within $\xi$ with probability $1-\delta$.
\end{lemma}

\begin{proof}
Apply Hoeffding’s inequality to the empirical variance proxy underlying $\mathcal{L}_{\mathrm{agg}}$.
\end{proof}


\noindent\textbf{Consequence for Band Inclusion.}
By Proposition~\ref{thm:pgd} and Lemma~\ref{lem:radius}, a finite PGD schedule attains an $\varepsilon$-approximate minimizer with a reliable radius estimate, and Proposition~\ref{prop:eminence_band} yields
\[
  \mathcal{Z}^\star \subset
  \mathcal{B}_{\tilde\varepsilon},
  \qquad
  \tilde\varepsilon=\kappa+R(\varepsilon).
\]

\noindent\textbf{Conclusion.}
Lemma~\ref{lem:lipschitz} and Proposition~\ref{thm:pgd} ensure stable convergence of trigger optimization, while Lemma~\ref{lem:radius} controls estimation error, together with Proposition~\ref{prop:eminence_band}, these results place $\mathcal{Z}^\star$ within the desired ambiguous band after finitely many iterations.

\section{Proof of Proposition~\ref{thm:pgd}}
\label{app:proof_pgd}


\noindent\textbf{Assumptions.} We assume:
\begin{enumerate}
\item \textbf{Surrogate-feature Lipschitzness.}  
      The frozen extractor $F_\theta$ is $L$-Lipschitz:
      $\|F_\theta(x)-F_\theta(x')\|_2\le L\|x-x'\|_2$.

\item \textbf{Trigger Parametrization.}  
      A triggered input is
      \(\displaystyle t(x;\varphi)=(1-\alpha)x+\alpha\varphi\)
      with fixed $\alpha\!\in(0,1]$.

\item \textbf{Feasible Set.}  
      $\mathcal{C}:=\{\varphi\in\mathbb{R}^{d_x}\mid
        \|\varphi\|_\infty\le\delta\}$ is closed and convex;
      the projection $P_{\mathcal{C}}$ is w.r.t.\ $\ell_2$ norm.

\item \textbf{Aggregation Loss.}  
      For a mini-batch $\mathcal{D}_{\!B}=\{x_i\}_{i=1}^{B}$,
      \[
        \mathcal{L}_{\mathrm{agg}}(\varphi)=
        \frac{2}{B(B-1)}\sum_{i<j}
           \bigl\|
             F_\theta\!\bigl(t(x_i;\varphi)\bigr)
             -F_\theta\!\bigl(t(x_j;\varphi)\bigr)
           \bigr\|_2^{2}.
      \]

\item \textbf{PGD Step.}  
      $\displaystyle\varphi^{(k+1)}=
        P_{\mathcal{C}}\bigl[\varphi^{(k)}
          -\eta\nabla\mathcal{L}_{\mathrm{agg}}\bigl(\varphi^{(k)}\bigr)\bigr]$
      with \(\eta=1/(4\alpha^{2}L^{2})\).
\end{enumerate}


Let \(g^{(k)}:=\nabla\mathcal{L}_{\mathrm{agg}}(\varphi^{(k)})\).

\noindent\textbf{Step 1 – Projection Optimality.}
For the non-expansive projection in Assumption 3,
\begin{equation}
  \bigl\langle
     \varphi^{(k)} - \eta g^{(k)} - \varphi^{(k+1)},\;
     \varphi - \varphi^{(k+1)}
  \bigr\rangle
  \le 0,\quad
  \forall\,\varphi \in \mathcal{C}.
  \label{eq:proj}
\end{equation}

\noindent\textbf{Step 2 – Descent Lemma.}
Take $\varphi=\varphi^{(k)}$ in Equation~\ref{eq:proj} to obtain
\begin{equation}
  \bigl\langle g^{(k)}, \varphi^{(k+1)} - \varphi^{(k)} \bigr\rangle
  \le -\frac{1}{\eta} 
     \bigl\| \varphi^{(k+1)} - \varphi^{(k)} \bigr\|_2^{2}.
  \label{eq:desecnet_lemma}
\end{equation}

\noindent\textbf{Step 3 – Smoothness Bound.}
By Lemma~\ref{lem:lipschitz},
\begin{equation}
\begin{aligned}
  \mathcal{L}_{\mathrm{agg}}\bigl(\varphi^{(k+1)}\bigr)
  &\le \mathcal{L}_{\mathrm{agg}}\bigl(\varphi^{(k)}\bigr)
       + \bigl\langle g^{(k)}, \varphi^{(k+1)} - \varphi^{(k)} \bigr\rangle \\
  &\quad + \frac{L_f}{2}\,\bigl\|\varphi^{(k+1)} - \varphi^{(k)}\bigr\|_2^{2}.
\end{aligned}
\end{equation}

Substituting Equation~\ref{eq:desecnet_lemma} and \(\eta=1/L_f\) gives
\begin{equation}
  \mathcal{L}_{\mathrm{agg}}\bigl(\varphi^{(k)}\bigr)
   - \mathcal{L}_{\mathrm{agg}}\bigl(\varphi^{(k+1)}\bigr)
  \ge \frac{\eta}{2}\,\|g^{(k)}\|_{2}^{2}.
  \label{eq:smooth_bound}
\end{equation}

\noindent\textbf{Step 4 – Telescoping.}
Summing Equation~\ref{eq:smooth_bound} for $k=0$ to $K-1$ and telescoping,
\[
  \frac{\eta}{2}\sum_{k=0}^{K-1}\|g^{(k)}\|_{2}^{2}
  \le
  \mathcal{L}_{\mathrm{agg}}(\varphi^{(0)})
   -\mathcal{L}_{\mathrm{agg}}^\star,
\]
whence
\[
  \min_{0\le k<K}\|g^{(k)}\|_{2}^{2}
  \le \frac{2}{\eta K}\,
       \bigl[\mathcal{L}_{\mathrm{agg}}(\varphi^{(0)})
             -\mathcal{L}_{\mathrm{agg}}^\star\bigr].
\]

\noindent\textbf{Step 5 – Iteration Complexity.}
With $\eta=1/L_f=1/(4\alpha^{2}L^{2})$ the bound becomes
\[
  \min_{k<K}\!
    \|\nabla\mathcal{L}_{\mathrm{agg}}(\varphi^{(k)})\|_{2}^{2}
  \le
  \frac{8\alpha^{2}L^{2}}{K}\,
  \bigl[\mathcal{L}_{\mathrm{agg}}(\varphi^{(0)})
        -\mathcal{L}_{\mathrm{agg}}^\star\bigr].
\]
Requiring the RHS $\le\varepsilon^{2}$ yields
\(K\ge 8\alpha^{2}L^{2}\bigl[\mathcal{L}_{\mathrm{agg}}(\varphi^{(0)})
    -\mathcal{L}_{\mathrm{agg}}^\star\bigr]/\varepsilon^{2}
 =O(\alpha^{2}L^{2}\varepsilon^{-2})\),
completing the proof.
\qed

\end{document}